\documentclass{article}

\usepackage[final,nonatbib]{neurips_2022}

\usepackage[utf8]{inputenc} % allow utf-8 input
\usepackage[T1]{fontenc}    % use 8-bit T1 fonts
\usepackage{hyperref}       % hyperlinks
\usepackage{url}            % simple URL typesetting
\usepackage{booktabs}       % professional-quality tables
\usepackage{amsfonts}       % blackboard math symbols
\usepackage{nicefrac}       % compact symbols for 1/2, etc.
\usepackage{microtype}      % microtypography
\usepackage{xcolor}         % colors
\usepackage{graphicx}
\usepackage{subfigure}
\usepackage{amsmath}
\usepackage{amssymb}
\usepackage{amsthm}
\usepackage{mathrsfs}
\usepackage{algorithm}
\usepackage{algpseudocode}
\usepackage{verbatim}
\usepackage{cite}
\usepackage{hyperref}

\newtheorem{assumption}{\bf Assumption}
\newtheorem{theorem}{\bf Theorem}
\newtheorem{lemma}{\bf Lemma}

\newtheorem{definition}{\bf Definition}
\newtheorem{corollary}{\bf Corollary}
\newtheorem{proposition}{\bf Proposition}

\newcommand{\ew}[1]{\textcolor{blue}{#1}}
\newcommand{\zy}[1]{\textcolor{red}{#1}}

\title{A Communication-efficient Algorithm with Linear Convergence for Federated Minimax Learning}

\author{%
  Zhenyu Sun \\
  Department of Electrical and Computer Engineering\\
  Northwestern University\\
  Evanston, IL 60208 \\
  \texttt{zhenyusun2026@u.northwestern.edu} \\
  \And
  Ermin Wei \\
  Department of Electrical and Computer Engineering\\
  Northwestern University\\
  Evanston, IL 60208 \\
  \texttt{ermin.wei@northwestern.edu} \\
}

\begin{document}

\maketitle

\begin{abstract}
   In this paper, we study a large-scale multi-agent minimax optimization problem, which models many interesting applications in statistical learning and game theory, including Generative Adversarial Networks (GANs). The overall objective is a sum of agents' private local objective functions. We focus on the federated setting, where agents can perform local computation and communicate with a central server. Most existing federated minimax algorithms either require communication per iteration or lack performance guarantees with the exception of Local Stochastic Gradient Descent Ascent (SGDA), a multiple-local-update descent ascent algorithm which guarantees convergence under a diminishing stepsize. By analyzing Local SGDA under the ideal condition of no gradient noise, we show that generally it cannot guarantee exact convergence with constant stepsizes and thus suffers from slow rates of convergence. To tackle this issue, we propose FedGDA-GT, an improved Federated (Fed) Gradient Descent Ascent (GDA) method based on Gradient Tracking (GT). When local objectives are Lipschitz smooth and strongly-convex-strongly-concave, we prove that FedGDA-GT converges linearly with a constant stepsize to global $\epsilon$-approximation solution with $\mathcal{O}(\log (1/\epsilon))$ rounds of communication, which matches the time complexity of centralized GDA method. Then, we analyze the general distributed minimax problem from a statistical aspect, where the overall objective approximates a true population minimax risk by empirical samples. We provide generalization bounds for learning with this objective through Rademacher complexity analysis. Finally, we numerically show that FedGDA-GT outperforms Local SGDA. 
\end{abstract}

\section{Introduction}

In recent years, minimax learning theory has achieved significant success in attaching relevance to many modern machine learning and statistical learning frameworks, including Generative Adversarial Networks (GANs) \cite{GAN,ImprovedGAN,LSGAN}, reinforcement learning \cite{RL17}, adversarial training \cite{DuchiAT17, ADICLR18}, robust estimation and optimization \cite{RobustLinear13,Duchi16,Duchi17,Duchi18,Liang20}, and domain adaptation \cite{Zhao18,Mohri19}. Generally, a minimax learning problem is modeled as a game between two players with opposite goal, i.e., one minimizes the objective while the other maximizes it. 

Most current studies in the field of machine learning are targeted at understanding the minimax problem from the view of the speed of convergence and the accuracy of fixed points. In the centralized setting, gradient descent ascent (GDA), which is an extension of gradient descent (GD), stands out for its simple implementation. Specifically, at each iteration, the "min" player conducts gradient descent over its decision variable while the "max" player performs gradient ascent in contrast. Due to the huge volumn of data, stochastic gradient descent ascent (SGDA) is preferred in machine learning settings. Theoretical guarantees are well-established for GDA and SGDA in \cite{Nedic09,Jordan20}. However, in practical scenarios, concerns on computation efficiency and data privacy trigger the development of federated learning over a server-client topology and distributed learning over a general graph. These often require communication with the server or neighbors at each iteration \cite{TZhang14,Wang17, Mohri19} and inapplicable to scenarios where communication is expensive.

In this work, we focus on solving a minimax problem in Federated Learning (FL) setting with one server and multiple client/agents. %has emerged as a promising learning paradigm that effectively reduces communication cost, which also allows the non-i.i.d. property of local data. 
In FL, the server hands over computation burden to agents, which perform training algorithms on their local data. The local trained models are then reported to the server for aggregation. This process is repeated with periodic communication. Much existing literature in FL, however, focuses on minimization optimization \cite{FedAvg,Scaffold,FedNova,FedSplit,FedLin}. %Until recent years, applying the idea of federated learning to minimax problems is gradually coming into focus. 
The limited literature on federated minimax problems either lacks theoretical guarantees or requires frequent communication \cite{Mohri19,FedTGAN,BF-FedGAN}, with the exception of Local Stochastic Gradient Descent Ascent (SGDA) \cite{Deng21}. In Local SGDA, each agent (or client) performs multiple steps of stochastic gradient descent ascent before communication to the server, which then aggregates local models by averaging. Under careful selection of diminishing learning rates, \cite{Deng21,Gauri22} show that Local SGDA converges to the global optimal solution sub-linearly. %Their convergence guarantees are then improved by \cite{}. 
However, as we show here, when we try to improve the speed of convergence and reduce communication overhead by introducing a constant stepsize to Local SGDA, it fails to converge to the exact optimal solution, even when full gradients are used. % Particularly, we show that %assuming strong-convexity-strong-concavity and Lipschitz smoothness on local objectives, Local SGDA converges linearly, but to an incorrect fixed point,
%Therefore, diminishing stepsizes are necessary for Local SGDA to guarantee exact convergence, which leads to a slow convergence rate and hence more communication rounds. 
  Thus Local SGDA can either be fast(with little communication)-but-inaccurate or accurate-but-slow(with much communication). 
 To address this tradeoff between model accuracy and communication efficiency, we develop FedGDA-GT, Federated Gradient Descent Ascent based on Gradient Tracking (FedGDA-GT), and show that it can achieve fast linear convergence  while preserving accuracy.

%It is worth noting that algorithm design focusing on convergence speed and optimality for minimax problems is not the only factor that contributes to the success of learning tasks. In this paper, 
In addition to solving the minimax problem, we also study the generalization performance  of distributed minimax problems in statistical learning, which measures the influence of sampling on the trained model. In most existing works, generalization analysis is established in the context of empirical risk minimization (ERM) \cite{Loh11,FML18}. It is well-known that for generic loss functions, the learning error for centralized ERM is on the order of $\mathcal{O}(1/\sqrt{N})$ with $N$ denoting the total number of training samples. Recently, several works derive generalization bounds on centralized minimax learning problems with the same order \cite{Tse16,JLee18,Farnia21}. For generalization analysis in distributed minimax learning, learning bounds are only provided for specific scenarios, e.g., agnostic federated learning \cite{Mohri19} and multiple-source domain adaptation \cite{Zhao18}. In this paper, we provide generalization bounds for distributed empirical minimax learning with the same order as results of centralized cases, generalizing the results in \cite{Mohri19}.  

\begin{comment}

\begin{itemize}
    %\setlength{\itemindent}{0em}
    %\setlength{\itemsep}{0ex}
    \item Analyzing the generalization properties of empirical minimax learning in distributed settings by allowing different local distributions, 
    \item Characterizing the behavior of fixed points of Local SGDA, which reveals the impact of objective heterogeneity and multiple local updates on the model accuracy, 
    \item Resolving the conflict between model accuracy and communication efficiency by developing a linear-rate federated minimax algorithm that guarantees exact convergence,
    \item Providing numerical results which suggest communication efficiency of our algorithm compared to Local SGDA and centralized GDA.
\end{itemize}

\end{comment}
\subsection{Related work}

\paragraph{Centralized minimax learning}
Historically, minimax problems have gained attraction of researchers since several decades ago. An early instantiation is bilinear minimax problem, which becomes a milestone in game theory together with von Neumann's theorem \cite{vNeumann28}. A simple algorithm is then proposed to solve this bilinear problem efficiently \cite{Robinson51}. \cite{Sion58} generalizes von Neumann's theorem to convex-concave games, which triggers an explosion in algorithmic research \cite{EG76,Nesterov,OGDA20}. GDA, as one widely used algorithm, is notable for its simple implementation. It is well-known that GDA can achieve an $\epsilon$-approximation solution with $\mathcal{O}(\log(1/\epsilon))$ iterations for strongly-convex-strongly-concave games and with $\mathcal{O}(\epsilon^{-2})$ iterations for convex-concave games under diminishing stepsizes \cite{Nedic09}. Very recently, nonconvex-nonconcave minimax problems appear to be a main focus in optimization and machine learning, due to the emergence of GANs. Several related works are listed therein \cite{VIGAN18,Liu19,Lee19,Liu20,DDJ21}.

\paragraph{Distributed and federated minimax learning}
A few recent studies are devoted to distributed minimax problems due to the increasing volume of data and concerns on privacy and security. Algorithm design and convergence behaviors are extensively studied for minimax problems in the context of distributed optimization, where communication is required at each iteration \cite{Cortes15,BSG20,Xian21,Rogozin21}. In the federated setting, \cite{DRO-FedAvg} proposes a multiple-local-update algorithm to deal with distribution shift issue. \cite{Ali20} studies federated adversarial training under nonconvex-PL objectives. FedGAN is proposed in \cite{FedGAN} to train GANs in a communication-efficient way. However, these works are targeted at some specific scenarios. Very recently, aiming to solve the general federated minimax problems, \cite{Deng21} proposes Local SGDA by allowing each agent performing multiple steps of GDA before communication. The authors also prove sub-linear convergence for Local SGDA under diminishing stepsizes. Their convergence guarantees is then improved by \cite{Gauri22} to match the results of centralized SGDA \cite{Jordan20}. However, we note that all these algorithms require diminishing learning rates to obtain exact solutions, which suffer from relatively slow convergence speed, but our algorithm allows constant stepsizes and hence linear convergence can be achieved.

\paragraph{Generalization of minimax learning}
Recently, generalization properties of minimax learning problems have been widely studied in different scenarios, including GANs and adversarial training.
For GANs, \cite{Bai18} analyzes the generalization performance of GANs when discriminators have restricted approximability. \cite{Zhang-bounds18} evaluates generalization bounds under different metrics. In contrast, \cite{Arora17} suggests a dilemma about GANs' generalization properties. 
In the context of adversarial training, generalization performances are studied through Rademacher complexity analysis \cite{Yin19,Attias19}, function transformation \cite{Loh18}, margin-based \cite{Wei19} approaches. \cite{Zhang-general-bound21} studies the generalization bounds of convex-concave objective functions with Lipschitz continuity. 
%In the context of adversarial training, \cite{Yin19} investigates generalization performance of trained models under $l_{\infty}$ attacks by Rademacher complexity analysis. \cite{Loh18} derives learning bounds for adversarial risks via function transformation. \cite{Wei19} builds the relationship between all-layer margin and generalization for deep models. In \cite{Attias19}, the authors show generalization bounds for robust regression and classification problems. \cite{Zhang-general-bound21} studies the generalization bounds of convex-concave objective functions with Lipschitz continuity. 
However, all these works are under centralized setting. Recently, \cite{Mohri19} provides generalization analysis under agnostic federated learning, where the objective is optimized for any target distribution formed by a mixture of agents' distributions. Our work extends their generalization analysis to general distributed minimax learning problems.

\subsection{Our Contributions.} We summarize our main contributions as follows: (1) In federated setting, characterizing the behavior of fixed points of Local SGDA, which reveals the impact of objective heterogeneity and multiple local updates on the model accuracy [see Section \ref{sec_lsgda}]; (2) Resolving the tradeoff between model accuracy and communication efficiency by developing a linear-rate federated minimax algorithm that guarantees exact convergence [see Section \ref{sec_fedgda-gt}]; (3) Analyzing the generalization properties of empirical minimax learning in distributed settings through Rademacher complexity analysis [see Section \ref{Sec_bound}]; (4) Providing numerical results which suggest communication efficiency of our algorithm compared to Local SGDA and centralized GDA [see Section \ref{Sec_exp}].

\textbf{Notations.} In this paper, we let $\Vert \cdot \Vert$ denote $l_2$-norm  and $\vert \cdot \vert$ denote the cardinality of a set or a collection, or absolute value of a scalar. Vectors are column vectors by default and $z=(x,y)$ forms  the concatenated vector with $z = [x^T, y^T]^T$. Vectors and scalars for agent $i$ are denoted using subscript $i$, e.g., $f_i(x,y)$. Superscripts, e.g., $t$, denote the indices of iterations. We let the gradient of $f(x,y)$ by $\nabla f(x,y) = (\nabla_x f(x,y), \nabla_y f(x,y))$, where $\nabla_x f(x,y)$ and $\nabla_y f(x,y)$ denote the gradients with respect to $x$ and $y$, respectively. 
%%%%%%%%%%%%%%%%%%%%%%%%%%%%%%%%%%%%%%%%%%%%%%%%%%%%%%%%%%%%%%%
%%%%%%%%%%%%%%%%%%%%%%%%%%%%%%%%%%%%%%%%%%%%%%%%%%%%%%%%%%%%%%%
%%%%%%%%%%%%%%%%%%%%%%%%%%%%%%%%%%%%%%%%%%%%%%%%%%%%%%%%%%%%%%%
%%%%%%%%%%%%%%%%%%%%%%%%%%%%%%%%%%%%%%%%%%%%%%%%%%%%%%%%%%%%%%%
\section{Problem Setup}

In this paper, we consider the general constrained minimax distributed optimization problem collectively solved by $m$ agents:
\begin{equation}    \label{global_emp_risk}
    \min_{x \in \mathcal{X}} \max_{y \in \mathcal{Y}} \left\{ f(x, y) := \frac{1}{m} \sum_{i=1}^m f_i(x, y) \right\},
\end{equation}
where $\mathcal{X}, \mathcal{Y}$ are some compact feasible sets contained in $\mathbb{R}^p$ and $\mathbb{R}^q$, $x \in \mathcal{X}$ is a $p$-dimension vector, $y \in \mathcal{Y}$ is a $q$-dimension vector and $f_i(\cdot,\cdot)$ is the local objective function of agent $i$, $\forall i=1\dots,m$.

Solving \eqref{global_emp_risk} is equivalent to finding a minimax point of $f(x,y)$, defined as follows:
\begin{definition}  \label{def_minimax-point}
    The point $(x^*, y^*)$ is said to be a minimax point of $f(x,y)$ if 
    $$ f(x^*, y) \le f(x^*, y^*) \le \max_{y' \in \mathcal{Y}}f(x, y'), \forall x \in \mathcal{X}, y \in \mathcal{Y} .$$
\end{definition}

%\begin{definition}  \label{def_saddle-point}
 %   The point $(x^*, y^*)$ is said to be a saddle point of $f(x,y)$ if 
  %  $$ f(x^*, y) \le f(x^*, y^*) \le f(x, y^*), \forall x \in \mathcal{X}, y \in \mathcal{Y} .$$
%\end{definition}

The first-order necessary condition for minimax points is given by the following lemma. 
\begin{lemma}   \cite{Jin-minimax}  \label{Lmm_stationary-point}
    Assume $f$ is continuously differentiable. Then, any minimax point $(x^*, y^*)$ in the interior of $\mathcal{X}\times \mathcal{Y}$ satisfies
    $$
        \nabla_x f(x^*, y^*) = \nabla_y f(x^*, y^*) = 0.
    $$
\end{lemma}

%%%%%%%%%%%%%%%%%%%%%%%%%%%%%%%%%%%%%%%%%%%%%%%%%%%%%%%%%%%%%%%
%%%%%%%%%%%%%%%%%%%%%%%%%%%%%%%%%%%%%%%%%%%%%%%%%%%%%%%%%%%%%%%
%%%%%%%%%%%%%%%%%%%%%%%%%%%%%%%%%%%%%%%%%%%%%%%%%%%%%%%%%%%%%%%
%%%%%%%%%%%%%%%%%%%%%%%%%%%%%%%%%%%%%%%%%%%%%%%%%%%%%%%%%%%%%%%
\section{FedGDA-GT: A linear-rate algorithm for federated minimax learning}    \label{Sec_algorithm}
%From the generalization bound in the last section, it guarantees that the optimal solution obtained by solving the empirical minimax problem generally has a good performance on the population minimax problem if the sample size is large enough. Thus, it is reasonable to expect that by solving \eqref{global_emp_risk} a relatively good model can be obtained. 

In this section, we focus on solving \eqref{global_emp_risk} in the federated setting, where the $m$ agents are connected to a central server. In general, the agents' communication with the server is more expensive than local computation. We show that the existing methods either require lots of communication (i.e. the convergence rate is sublinear) or could only converge to an inexact solution with linear rates. This suggests a tradeoff between model accuracy and communication efficiency. Motivated by this phenomenon, we aim at developing a communication-efficient minimax algorithm with linear convergence that preserves model accuracy and low communication overhead simultaneously.

We adopt the following standard assumptions on the problem. 
\begin{comment}

\begin{definition}  \label{def_SCSC}
    $f(x,y)$ is said to be $\mu$-strongly-convex-strongly-concave if $f(x,y)$ is $\mu$-strongly convex in $x \in \mathcal{X}$, i.e., given any $y \in \mathcal{Y}$,
    $$
        f(z, y) \ge f(x, y) + \langle \nabla_x f(x,y), z-x \rangle + \frac{\mu}{2}\Vert z-x \Vert^2, ~ \forall z,x \in \mathcal{X}
    $$
    and $f(x,y)$ is $\mu$-strongly concave in $y \in \mathcal{Y}$, i.e., given any $x \in \mathcal{X}$,
    $$
        f(x, z) \le f(x, y) + \langle \nabla_y f(x,y), z-y \rangle - \frac{\mu}{2}\Vert z-y \Vert^2, ~ \forall z,y \in \mathcal{Y}.
    $$
    where $\mathcal{X}$ and $\mathcal{Y}$ are compact and convex.
\end{definition}

\end{comment}

\begin{assumption}[$\mu$-strongly-convex-strongly-concave]  \label{assp_SCSC}
     For any $i=1,\dots,m$, $f_i(x,y)$ is twice-differentiable and is $\mu$-strongly-convex-strongly-concave with some $\mu > 0$ for any $(x,y) \in \mathbb{R}^p \times \mathbb{R}^q$, i.e., 
      $$ \mbox{for any given $y$,}\qquad
        f_i(z, y) \ge f_i(x, y) + \langle \nabla_x f_i(x,y), z-x \rangle + \frac{\mu}{2}\Vert z-x \Vert^2, ~ \forall z,x,
    $$
    $$  \mbox{for any given $x$,}\qquad
        f_i(x, z) \le f(x, y) + \langle \nabla_y f_i(x,y), z-y \rangle - \frac{\mu}{2}\Vert z-y \Vert^2, ~ \forall z,y.
    $$
\end{assumption}

\begin{assumption} [$L$-smoothness] \label{assp_smooth}
    There exists some $L>0$ such that for any $i=1,\dots,m$, $\Vert \nabla f_i(x,y) - \nabla f_i(x',y') \Vert \le L\Vert (x,y) - (x',y') \Vert, \forall (x,y), (x',y') \in \mathbb{R}^p \times \mathbb{R}^q$.
\end{assumption}

We note that although each $f_i(x,y)$ may have different $\mu_i$ and $L_i$, we can set $\mu = \min_{i=1,\dots,m}\mu_i$ and $L = \max_{i=1\dots,m}L_i$ to ensure Assumptions \ref{assp_SCSC} and \ref{assp_smooth} hold.

\subsection{Analysis on Local SGDA}\label{sec_lsgda}
We first study Local Stochastic Gradient Descent Ascent (SGDA) proposed in \cite{Deng21}, which is the only known method with convergence guarantees and can utilize multiple local computation steps before communicating with the central server for solving general federated minimax problems. In particular, in each iteration of Local SGDA, each agent updates its local model $(x_i, y_i)$ for $K$ times by using local stochastic gradients before communication, and then sends its local model to the server, which then computes the average of local models and sends it back.  

In \cite{Deng21}, the authors prove that under Assumptions \ref{assp_SCSC}-\ref{assp_smooth}, with bounded variance assumption on  all local gradients and vanishing learning rate, after $T$ rounds of communication, the convergence result of Local SGDA is
\begin{equation}    \label{convergence-SGDA}
    \mathbb{E}\left[ \Vert x^{(T)} - x^* \Vert^2 + \Vert y^{(T)} - y^* \Vert^2 \right] \le \mathcal{O}\left( T^{-1} \right) %+ \mathcal{O}\left( \mu^{-2}m^{-1} T^{-1} \right) 
    + \mathcal{O}\left( T^{-3} \right).
\end{equation}
This shows the slow sublinear convergence rate of this method, which translates to large amount of communication. We next consider an ideal deterministic version of Local SGDA (Algorithm \ref{alg_FedGDA}), where local stochastic gradients are replaced with full gradients.   %where $a = \max \left\{ 2048\kappa^2 K, 1024\sqrt{2}\kappa^2 K, 256\kappa^2 \right\}$, $K = \sqrt{T/m}$ is the number of local updates and $\kappa = L/\mu$ is the condition number.

%The convergence result \eqref{convergence-SGDA} states that under careful selection of stepsizes and the number of local update steps, Local SGDA converges in a sub-linear rate with bounded heterogeneity between agents. Overall, this sub-linear rate is not good enough, especially when communication round becomes large, the updates will be very slow. Thus, an accurate model basically requires many communication resources. Actually, diminishing stepsizes are necessary for Local SGDA to converge to the correct saddle point even in deterministic cases (see Proposition \ref{Thm_FedGDA-fixed-point}).

\begin{algorithm}
    \caption{Local SGDA} \label{alg_FedGDA}
    \begin{algorithmic}[1]
        \Require $(x^0, y^0)$ as initialization of the server
        \For {$t=0,1,\dots,T$}      % loop for the server
            \State $x^{t+1}_{i,0} = x^{t}$~, \quad $y^{t+1}_{i,0} = y^{t}, ~\forall i=1,\dots,m$\qquad
            \For {$k=0,1,\dots,K-1$} (in parallel for all agents)
                % local update
%                \State compute gradients $\nabla_{x}f_i(x_{i,k}^{t+1}, y_{i,k}^{t+1})$ and $\nabla_{y}f_i(x_{i,k}^{t+1}, y_{i,k}^{t+1})$ evaluated at $(x_{i,k}^{t+1}, y_{i,k}^{t+1})$
                \State $x_{i, k+1}^{t+1} = x_{i, k}^{t+1} - \eta_x \nabla_{x}f_i(x_{i,k}^{t+1}, y_{i,k}^{t+1}) $
                \State $y_{i, k+1}^{t+1} = y_{i, k}^{t+1} + \eta_y \nabla_{y}f_i(x_{i,k}^{t+1}, y_{i,k}^{t+1}) $ 
            \EndFor
            % server update
            \State $x^{t+1} =  \frac{1}{m}\sum_{i=1}^m x^{t+1}_{i,K} $~,
            \quad $y^{t+1} =  \frac{1}{m}\sum_{i=1}^m y^{t+1}_{i,K} $
        \EndFor
        \Ensure $(x^T, y^T)$ given by the server
    \end{algorithmic}
\end{algorithm}

In order to rewrite Local SGDA in a concise form, we define the operator for gradient descent evaluated at $(x,y)$ for agent $i$ as
$$
    \mathcal{D}_i^1(x,y) = x - \eta_x \nabla_x f_i(x,y), ~~\forall i=1,\dots,m.
$$
Given some $k\ge 1$, let $\mathcal{D}_i^k$ define the composition of $k$ $\mathcal{D}_i^1$ operators. Moreover, $\mathcal{D}_i^0 (x,y) = x$ is the identity operator on the first argument. Let $\mathcal{A}_i^1(x,y) = y + \eta_y \nabla_y f_i(x,y)$ be the gradient ascent operator with $\mathcal{A}_i^k (x,y)$ and $\mathcal{A}_i^0 (x,y) = y$ defined similarly. Then, given some initial point $(x^0, y^0)$, Algorithm \ref{alg_FedGDA} is rewritten as the recursion of the following for $t=0,1,\dots,T$:
\begin{eqnarray}    
    \tilde{x}_i^{t+1} = \mathcal{D}_i^K(x^t, y^t)~, ~~~~\tilde{y}_i^{t+1} = \mathcal{A}_i^K(x^t, y^t)~, ~~ \forall i=1,\dots,m   \nonumber   
    %\tilde{y}_i^{t+1} &=& \mathcal{B}_i^K(x^t, y^t), ~~ \forall i=1,\dots,m   \nonumber   
    %y^{t+1} &=& \frac{1}{m}\sum_{i=1}^m \tilde{y}_i^{t+1}
\end{eqnarray}
\begin{equation}    \label{FedGDA_operator}
    x^{t+1} = \frac{1}{m}\sum_{i=1}^m \tilde{x}_i^{t+1}~,~~~~   y^{t+1} = \frac{1}{m}\sum_{i=1}^m \tilde{y}_i^{t+1} .
\end{equation}
The following result shows that Local SGDA generally cannot guarantee the convergence to the optimal solution of \eqref{global_emp_risk} with fixed stepsizes even with full gradients.
\begin{proposition} \label{Thm_FedGDA-fixed-point}
    Suppose $f_i$ is differentiable, $\forall i=1,\dots,m$. For any $K\ge1$, let $\{ (x^t, y^t) \}$ be the sequence generated by \eqref{FedGDA_operator} (or equivalently by Algorithm \ref{alg_FedGDA}). If $\{ (x^t, y^t) \}$ converges to a fixed point $(x^*, y^*)$, then
    \begin{eqnarray}
        \frac{1}{m}\sum_{i=1}^m \sum_{k=0}^{K-1} \nabla f_i(\mathcal{D}_i^k(x^*, y^*), \mathcal{A}_i^k(x^*, y^*)) = 0.    \nonumber 
%        \frac{1}{m}\sum_{i=1}^m \sum_{k=0}^{K-1} \nabla_{y} f_i(\mathcal{D}_i^k(x^*, y^*), \mathcal{A}_i^k(x^*, y^*)) &=& 0    \nonumber
    \end{eqnarray}
\end{proposition}
Note that when $K=1$, Local SGDA reduces to GDA and the fixed point $(x^*, y^*)$ satisfies
\begin{equation}
    \frac{1}{m}\sum_{i=1}^m  \nabla_{x} f_i(x^*, y^*) = \frac{1}{m}\sum_{i=1}^m  \nabla_{y} f_i(x^*, y^*) = 0    \nonumber  
\end{equation}
which meets the optimality condition (by Lemma \ref{Lmm_stationary-point}) of the minimax point of $f(x,y)$. For any other $K$, the fixed point characterization is different from the minimax point. An illustrative example is provided in Appendix \ref{Apx_example}.

\begin{comment}
\ew{This argument is incorrect, as stepsizes vanish, in order to guarantee (x,y) do not move too much, the gradients also need to shrink/not grow faster than stepsizes. The previous argument is only about fixed point, not necessarily for large $t$. Also there is not much value added here, suggest to skip or shorten this discussion. }
On the other hand, Proposition \ref{Thm_FedGDA-fixed-point} also explains \zy{why diminishing stepsizes are necessary to guarantee correct fixed points}. Roughly speaking, \ew{if the sequence converges to a fixed point asymptotically, then} as iteration index $t$ becomes large, $x^t \approx x^*$, $y^t \approx y^*$ and the stepsizes $\eta_x$ and $\eta_y$ vanish towards zeros. 
Then, for $k \ge 1$, we have $\mathcal{D}_i^k (x^t, y^t) \approx x^t$ and $\mathcal{A}_i^k (x^t, y^t) \approx y^t$ for any $i=1\dots,m$, i.e., multiple local gradient operations do not cause $(x^t, y^t)$ moving far away from $(x^*, y^*)$. This indicates that
\begin{eqnarray}
    0 = \frac{1}{m}\sum_{i=1}^m \sum_{k=0}^{K-1} \nabla_{x} f_i(\mathcal{D}_i^k(x^*, y^*), \mathcal{A}_i^k(x^*, y^*)) &\approx& \frac{1}{m}\sum_{i=1}^m  \nabla_{x} f_i(x^*, y^*)    \nonumber   \\
    0 = \frac{1}{m}\sum_{i=1}^m \sum_{k=0}^{K-1} \nabla_{y} f_i(\mathcal{D}_i^k(x^*, y^*), \mathcal{A}_i^k(x^*, y^*)) &\approx& \frac{1}{m}\sum_{i=1}^m  \nabla_{y} f_i(x^*, y^*)    \nonumber
\end{eqnarray}
\end{comment}

  We mention that while having multiple local gradient steps in the inner loop allows the agents to reduce the frequency of communication, it also pulls the agents towards their local minimax points as opposed to the global one. Motivated by this observation, we next propose a modification involving a gradient correction term, also known as gradient tracking (GT), to make sure the agents are moving towards the global minimax point.

\subsection{FedGDA-GT and  convergence guarantees}\label{sec_fedgda-gt}
As indicated in Section \ref{sec_lsgda}, Local SGDA cannot guarantee exact convergence under fixed stepsizes, which implies a tradeoff between communication efficiency and model accuracy. Aiming at resolving this issue, in this section, we introduce our algorithm FedGDA-GT, formally described in Algorithm \ref{alg_FedLC}, that can reach the optimal solution of \eqref{global_emp_risk} with fixed stepsizes, full gradients and the following assumption. 
\begin{assumption}  \label{assp_convex-sets}
    The feasible sets $\mathcal{X}$ and $\mathcal{Y}$ are compact and convex. Moreover, there is at least one minimax point of $f(x,y)$ lying in $\mathcal{X}\times \mathcal{Y}$.
\end{assumption}
\begin{algorithm}
    \caption{FedGDA-GT} \label{alg_FedLC}
    \begin{algorithmic}[1]
        \Require $(x^0, y^0)$ as initialization of the global model    
        \For {$t=0,1,\dots,T$}      % loop for the server
            \State Server broadcasts $(x^t, y^t)$
            \State Agents compute $(\nabla_x f_i(x^t,y^t), \nabla_y f_i(x^t,y^t))$ and send it to the server
            \State Server computes $(\nabla_x f(x^t,y^t), \nabla_y f(x^t,y^t))$ and broadcasts it
            \State Each agent $i$ for $i=1,\dots,m$ sets
 $x^{t+1}_{i,0} = x^{t},~ y^{t+1}_{i,0} = y^{t}$
%            \State compute $\nabla_x f(x^t, y^t)$, $\nabla_y f(x^t, y^t)$ and send them to all agents
            \For {$k=0,1,\dots,K-1$} (in parallel for all agents)
                % local update
%                \State compute gradients $\nabla_{x}f_i(x_{i,k}^{t+1}, y_{i,k}^{t+1})$,$\nabla_{y}f_i(x_{i,k}^{t+1}, y_{i,k}^{t+1})$ and $\nabla_{x}f_i(x^t, y^t)$,$\nabla_{y}f_i(x^t, y^t)$
                \State $x_{i, k+1}^{t+1} = x_{i, k}^{t+1} - \eta (\nabla_{x}f_i(x_{i,k}^{t+1}, y_{i,k}^{t+1}) - \nabla_{x}f_i(x^{t}, y^{t}) + \nabla_{x}f(x^t, y^t))$
                \State $y_{i, k+1}^{t+1} = y_{i, k}^{t+1} + \eta (\nabla_{y}f_i(x_{i,k}^{t+1}, y_{i,k}^{t+1}) - \nabla_{y}f_i(x^{t}, y^{t}) + \nabla_{y}f(x^t, y^t))$
            \EndFor
            \State All agents send $(x_{i,K}^{t+1}, y_{i,K}^{t+1})$ to the server to compute $(x^{t+1},y^{t+1})$ by
            % server update
            \State $x^{t+1} = \mathrm{Proj}_{\mathcal{X}} \left( \frac{1}{m}\sum_{i=1}^m x^{t+1}_{i,K} \right), ~~ y^{t+1} = \mathrm{Proj}_{\mathcal{Y}} \left( \frac{1}{m}\sum_{i=1}^m y^{t+1}_{i,K} \right)$
        \EndFor
        \Ensure $(x^T, y^T)$ given by the server
    \end{algorithmic}
\end{algorithm}

$\mathrm{Proj}_\mathcal{X}(\cdot)$ and $\mathrm{Proj}_\mathcal{Y}(\cdot)$ denote the projection operators, i.e.,
$$
    \mathrm{Proj}_{\mathcal{X}}(x) = \arg\min_{z \in \mathcal{X}} \Vert z - x \Vert ~~\mathrm{and} ~~ \mathrm{Proj}_{\mathcal{Y}}(y) = \arg\min_{z \in \mathcal{Y}} \Vert z - y \Vert,
$$
which are well defined by the previous assumption and are used to guarantee the output of Algorithm \ref{alg_FedLC} is feasible. 
%We explain how Algorithm \ref{alg_FedLC} works: at the beginning, the server randomly chooses some initial point $(x^0, y^0)$ in the feasible product space $\mathcal{X} \times \mathcal{Y}$ and sends it to all agents. After computing $\nabla_x f_i(x^0, y^0)$ and $\nabla_y f_i(x^0, y^0)$, each agent sends its local gradients evaluated at $(x^0, y^0)$ back to the server which then gets $\nabla_x f(x^0, y^0)$ and $\nabla_y f(x^0, y^0)$ by averaging. In the following every iteration $t$, the server first sends its current model $(x^t, y^t)$ and corresponding global gradient $(\nabla_x f(x^t, y^t), \nabla_y f(x^t, y^t))$ to all the agents for which to update in parallel local models using local gradients and correction terms $\nabla_x f(x^t, y^t)-\nabla_x f_i(x^t, y^t)$, $\nabla_y f(x^t, x^t)-\nabla_y f_i(x^t, y^t)$ for $K$ times. After all agents complete local updates, the server gathers local models for an averaging synchronization and projects it onto the feasible space. It is worth noting that only one round of communication is needed at every iteration $t$ but both local and global models and gradients are in exchange.
% %\ew{
% FedGDA-GT has an outer-loop (indexed by $t$) and an inner-loop (indexed by $k$). At the beginning of each outer iteration, the server broadcasts the current model $(x^t, y^t)$ to each agent and the agents compute $(\nabla_x f_i(x^t,y^t), \nabla_x f_i(x^t,y^t))$ and send the information back to the server. The server then aggregates the local information to compute the global gradient and broadcasts it to the agents. The agents then } 
The main differences between FedGDA-GT and Local-SGDA are the correction terms of $ - \nabla_{x}f_i(x^{t}, y^{t}) + \nabla_{x}f(x^t, y^t)$ and  $- \nabla_{y}f_i(x^{t}, y^{t}) + \nabla_{y}f(x^t, y^t)$, which track the differences between local and global gradients. By including the gradient correction terms, when $(x^{t+1}_{i,k}, y^{t+1}_{i,k})$ is not too far from $(x^t, y^t)$, we can expect $\nabla_{x}f_i(x_{i,k}^{t+1}, y_{i,k}^{t+1}) \approx \nabla_{x}f_i(x^{t}, y^{t})$ and similarly for $y$ gradients. Hence the updates reduce to simply taking global gradient descent and ascent steps, which coincide with centralized GDA updates and thus would have the correct fixed points. We next show the convergence guarantee of FedGDA-GT. 
 
\begin{comment}
We provide the intuition behind FedGDA-GT (Algorithm \ref{alg_FedLC}): in the centralized GDA, the algorithm alternates between gradient descent and gradient steps. %essentially updates its local model by $x_{i, k+1}^{t+1} = x_{i, k}^{t+1} - \eta \nabla_{x}f(x^{t+1}_{i,k}, y^{t+1}_{i,k})$, $y_{i, k+1}^{t+1} = y_{i, k}^{t+1} + \eta \nabla_{y}f(x^{t+1}_{i,k}, y^{t+1}_{i,k})$. 
However, this is impossible for in federated setting due to the lack of access to the global gradient information. To tackle this, each agent uses the model and gradient of the server $(x^t, y^t)$ and $(\nabla_x f(x^t, y^t), \nabla_y f(x^t, y^t))$ as a reference to guide its update. In this sense, by the correction techniques, the local update in Algorithm \ref{alg_FedLC} becomes the ideal version as centralized GDA and hence guarantees the correctness of the fixed point. 
\end{comment}

 We first establish the uniqueness of minimax point of $f(x,y)$ by the following lemma.
\begin{lemma}   \label{Lmm_unique-point}
    Under Assumptions \ref{assp_SCSC} and \ref{assp_convex-sets}, 
    $(x^*, y^*)$ is the unique minimax point of $f(x,y)$ if and only if 
    $$
        \nabla_x f(x^*, y^*) = \nabla_y f(x^*, y^*) = 0.
    $$
\end{lemma}

Next, we state the main convergence result of our algorithm.

\begin{theorem} \label{Thm_convergence}
    Suppose Assumptions \ref{assp_SCSC}, \ref{assp_smooth}, \ref{assp_convex-sets} are satisfied. Let $\{ (x^t, y^t) \}_{t=0}^{\infty}$ be a sequence generated by Algorithm \ref{alg_FedLC}. Then there exists a scalar $\eta_0 > 0$, such that for any stepsize $\eta \in (0,\eta_0)$
    $$
        \Vert x^t - x^* \Vert^2 + \Vert y^t - y^* \Vert^2 \le \rho(\eta)^t \left( \Vert x^0 - x^* \Vert^2 + \Vert y^0 - y^* \Vert^2 \right), \forall t = 0,1,\dots ,
    $$
    where $(x^*, y^*)$ is the unique minimax point of $f(x,y)$ and $\rho(\eta)$ is some scalar in $(0,1)$.
\end{theorem}
Theorem \ref{Thm_convergence} guarantees linear convergence of FedGDA-GT (Algorithm \ref{alg_FedLC}) to the correct optimal solution of \eqref{global_emp_risk} under suitable choices of stepsize $\eta$. We note that linear convergence here is with respect to the outer-loop (indexed by $t$), as the inner-loop (indexed by $k$) can be implemented cheaply without any communication. The fast convergence speed and exact convergence of this result eliminate the tradeoff between model accuracy and communication efficiency. Furthermore, no restriction on heterogeneity level of local objectives is placed. For the homogeneous setting, we show in Appendix \ref{Apx_homo} that the convergence rate of FedGDA-GT can be improved at least $K$ times, compared to heterogeneous setting.

\section{Generalization bounds on minimax learning problems} \label{Sec_bound}

In this section, we consider minimax statistical learning, an important application of the minimax framework. We view the problem \eqref{global_emp_risk}  in Section \ref{Sec_algorithm} as the empirical version of a population minimax  optimization  problem. We evaluate the generalization performance of the empirical problem using Rademacher complexity. 

To be more specific, in a minimax statistical learning task, each agent is allocated with some local dataset $\mathcal{S}_i = \{ \xi_{i,j} \}_{j=1}^{n_i}$, where $\xi_{i,j}$ denotes the $j$th sample of agent $i$ and $n_i$ denotes the number of local samples. $\mathcal{S} = \bigcup_{i=1}^m \mathcal{S}_i$ is the dataset with all samples. Moreover, we assume $n_i = n, \forall i=1,\dots,m$ and $N = mn$ is the total number of samples. Then, each local objective function $f_i(x, y)$ is defined by using local data, i.e.,
\begin{equation}    \label{local_emp_risk}
    f_i(x,y) = \frac{1}{n} \sum_{j=1}^n l(x,y;\xi_{i,j}),
\end{equation}
where $l(x,y;\xi_{i,j})$ is the loss function measured at point $\xi_{i,j}$ and $\eqref{local_emp_risk}$ is called local empirical minimax risk. Suppose that for agent $i$ each data sample $\xi_{i,j}$ is independently drawn from some underlying distribution $P_i$, denoted by $\xi \sim P_i$. We can further define the local population minimax risk as follows:
\begin{equation}    \label{local_pop_risk}
    R_i(x, y) = \mathbb{E}_{\xi \sim P_i}\left[ l(x,y; \xi) \right]
\end{equation}
and similarly, the global population minimax risk is defined by
\begin{equation}    \label{global_pop_risk}
    R(x,y) = \mathbb{E}_{\xi \sim P}\left[ l(x,y;\xi) \right],
\end{equation}
where $P$ is the underlying distribution of the whole dataset $\mathcal{S}$. Then, the minimax problem based on the population minimax risk is given by
\begin{equation}    \label{global_pop_minimax}
    \min_{x \in \mathcal{X}} \max_{y \in \mathcal{Y}} R(x,y).
\end{equation}

Sample applications of our formulation \eqref{local_emp_risk}-\eqref{global_pop_minimax} can be found in Appendix \ref{Apx_applications}.

While we are interested in calculating the population minimax risk, in practice we can only solve \eqref{global_emp_risk}, which is a sampled version of \eqref{global_pop_minimax}, since the true distribution $P$ is unknown. In this sense, how could we expect the optimal solution of \eqref{global_emp_risk} performs successfully on \eqref{global_pop_risk}? To answer this question, we provide the generalization bound which measures the performance of the model trained on the empirical minimax risk $f(x,y)$. Our generalization bound is based on the notion of Rademacher complexity \cite{FML18} defined by
\begin{equation}    \label{Rademacher}
    \mathscr{R}(\mathcal{X},y) = \mathbb{E}_{\xi \sim P} \mathbb{E}_{\sigma}\left[ \sup_{x \in \mathcal{X}} \frac{1}{mn}\sum_{i=1}^m \sum_{j=1}^n \sigma_{i,j} l(x,y;\xi_{i,j}) \right],
\end{equation}
where $\sigma = \{ \sigma_{i,j} \}, \forall i=1,\dots,m, \forall j=1,\dots,n$, is a collection of Rademacher variables taking values from $\{-1,1\}$ uniformly. Basically, the Rademacher complexity \eqref{Rademacher} captures the capability of $\mathcal{X}$ to fit random sign noise, i.e., $\sigma$. Note that $\sigma_{i,j}l(x,y;\xi_{i,j})$ measures how well the loss $l$ correlates with noise $\sigma$ on sample space $\mathcal{S}$. By taking the supremum, it means what the best extent the feasible set $\mathcal{X}$ on average can correlate with random noise $\sigma$. Thus, if $\mathcal{X}$ is richer, then the Rademacher complexity is bigger. Moreover, the minimax Rademacher complexity is defined on $\mathcal{Y}$ by
\begin{equation}    \label{minimax-Rademacher}
    \mathscr{R}(\mathcal{X},\mathcal{Y}) = \max_{y \in \mathcal{Y}} \mathscr{R}(\mathcal{X},y).
\end{equation}

Given $\epsilon > 0$, we further define the $\epsilon$-minimum cover of $\mathcal{Y}$ in $l_2$ distance by
\begin{equation}
    \mathcal{Y}_{\epsilon} = \arg\min_{C(\mathcal{Y}, \epsilon)} \big| C(\mathcal{Y}, \epsilon) \big|,  \nonumber
\end{equation}
where $C(\mathcal{Y}, \epsilon) = \left\{ B_{\epsilon}(y) : y \in \mathcal{Y} \right\}$ is a collection of open balls $B_{\epsilon}(y)$ centered at $y$ with radius $\epsilon$ such that for any $y \in \mathcal{Y}$ there exists some $B_{\epsilon}(y') \in C(\mathcal{Y}, \epsilon)$ with $y \in B_{\epsilon}(y')$. Note that $\big| \mathcal{Y}_{\epsilon} \big| < \infty$, since $\mathcal{Y}$ is compact and thus every open cover of $\mathcal{Y}$ has a finite subcover.

Then, we have the following generalization bound for the distributed minimax learning problem:
\begin{theorem} \label{Thm_general-bound}
    Suppose $\vert l(x,y;\xi) - l(x,y';\xi) \vert \le L_y \Vert y-y' \Vert$ and $\vert l(x,y;\xi) \vert \le M_i(y)$, $\forall x\in \mathcal{X}$, $\forall y,y' \in \mathcal{Y}_{\epsilon}$ and $\forall \xi \sim P_i, i=1\dots,m$ with some positive scalar $L_y$ and real-valued function $M_i(y)>0$. Then, given any $\epsilon > 0$ and $\delta > 0$, with probability at least $1-\delta$ for any $(x,y) \in \mathcal{X}\times \mathcal{Y}$,
    \begin{equation}    \label{eq_general-bound}
        R(x,y) \le f(x,y) + 2\mathscr{R}(\mathcal{X}, y) + \sqrt{\sum_{i=1}^m \frac{M_i^2(y)}{2m^2 n}\log \frac{\vert \mathcal{Y}_{\epsilon} \vert}{\delta}} + 2L_y \epsilon   .
    \end{equation}
    %(b). If $l(\cdot, y; \xi)$ is $\mu$-strongly convex $\forall y \in \mathcal{Y}$ and $G_i$-Lipschitz continuous $\forall \xi \in \mathcal{D}_i$, with probability at least $1-\delta$,
\end{theorem}

Theorem \ref{Thm_general-bound} generally states that given any $x,y$ feasible, it is highly possible that the distance between the global population minimax risk $R(x,y)$ and the global empirical minimax risk $f(x,y)$ can be bounded by Rademacher complexity and a term related to the number of agents and local sample size. Moreover, we allow the upper bound of $l(\cdot, y)$ can depend on different choices of $y$ and agents since different local distributions always have different effects on the value of $l(\cdot, y)$. 

Based on Theorem \ref{Thm_general-bound}, we further derive the high-probability bound on $\max_{y \in \mathcal{Y}}R(x,y)$.
\begin{corollary}   \label{Coro_max-bound}
    %Suppose $\vert l(x,y;\xi) - l(x,y';\xi) \vert \le L_y \Vert y-y' \Vert$ and $\vert l(x,y;\xi) \vert \le M_i(y)$, $\forall x\in \mathcal{X}$, $\forall y,y' \in \mathcal{Y}$ and $\forall \xi \sim P_i, i=1\dots,m$ with some positive scalar $L_y$ and \ew{real-valued function} $M_i(y)>0$. 
    Under the same conditions of Theorem \ref{Thm_general-bound}, with probability at least $1-\delta$ for any $x \in \mathcal{X}$, the following inequality holds for any $\epsilon>0$, $\delta > 0$:
    \begin{equation}    \label{eq_max-bound}
        Q(x) \le g(x) + 2 \mathscr{R}(\mathcal{X},\mathcal{Y}) + \sqrt{\max_{y \in \mathcal{Y}} \left\{ \sum_{i=1}^m \frac{M_i^2(y)}{2m^2 n} \right\} \log \frac{\vert \mathcal{Y}_{\epsilon}  \vert}{\delta}} + 2L_y \epsilon
    \end{equation}
    where $Q(x) = \max_{y \in \mathcal{Y}}R(x,y)$, $g(x) = \max_{y \in \mathcal{Y}} f(x,y)$ are the worst-case population and empirical risks, respectively.
\end{corollary}

We can further bound $\mathscr{R}(\mathcal{X}, \mathcal{Y})$ when the loss function $l(x,y;\xi)$ takes finite number of values for any $(x,y) \in \mathcal{X} \times \mathcal{Y}$ and any data samples.

\begin{lemma}   \label{Lmm_complexity-bound}
    Suppose for any $y \in \mathcal{Y}$ and $i \in \{1, \dots, m \}$, $\vert l(\cdot,y;\cdot) \vert$ is bounded by $M_i(y)$ and takes finite number of values. Further, assume that the VC-dimension of $\mathcal{X}$ is $d$. Then, the following inequality holds:
    \begin{eqnarray}
        \mathscr{R}(\mathcal{X}, \mathcal{Y}) \le \sqrt{2d \max_{y \in \mathcal{Y}} \left\{ \sum_{i=1}^m \frac{M^2_i(y)}{m^2 n} \right\} \left( 1 + \log \frac{mn}{d}\right)} . \label{eq_complexity-max-bound}
    \end{eqnarray}
\end{lemma}

Corollary \ref{Coro_max-bound} measures the error between the worst-case population and empirical risks, that is $\max_{y \in \mathcal{Y}} R(x,y) - \max_{y \in \mathcal{Y}}f(x,y)$. A smaller error indicates a better performance of the model, which is trained empirically, generalized on the underlying distribution $P$. Combining Lemma \ref{Lmm_complexity-bound}, we know that this worst-case error can be bounded by some decreasing function with respect to sample size. Thus, in order to get a better generalization performance, one effective way is to draw more local samples for each agent.

It it worth noting that the generalization bound is related to the term $\sum_{i=1}^m M_i^2(y)$ as shown in \eqref{eq_general-bound}-\eqref{eq_complexity-max-bound}, which essentially measures the effect of feasible set $\mathcal{Y}$ by means of different local data distributions on each agent. This also captures the heterogeneity of agents.

In fact, the bounds \eqref{eq_general-bound}-\eqref{eq_complexity-max-bound} we proposed are generalized versions of those in \cite{Mohri19}, which essentially include them as special cases by selecting suitable $M_i(y)$. Specifically, noting that in \cite{Mohri19}, the global population risk takes the form of $R(x,y) = \sum_{i=1}^m y_i R_i(x)$ with $y_i \ge 0$ and $\sum_{i=1}^m y_i = 1$, then by choosing $M_i(y) = m y_i M$ with some $M > 0$ we recover the same result. Moreover, compared to \cite{Tse16,JLee18}, where only centralized minimax learning problems are considered, our bounds have the same order of complexity $\mathcal{O}(1/\sqrt{N})$ with $N$ denoting the total sample size by taking the uniform bound on $l(\cdot)$ for all agents.

%%%%%%%%%%%%%%%%%%%%%%%%%%%%%%%%%%%%%%%%%%%%%%%%%%%%%%%%%%%%%%%%%
%%%%%%%%%%%%%%%%%%%%%%%%%%%%%%%%%%%%%%%%%%%%%%%%%%%%%%%%%%%%%%%%%
%%%%%%%%%%%%%%%%%%%%%%%%%%%%%%%%%%%%%%%%%%%%%%%%%%%%%%%%%%%%%%%%%
%%%%%%%%%%%%%%%%%%%%%%%%%%%%%%%%%%%%%%%%%%%%%%%%%%%%%%%%%%%%%%%%%
\section{Experiments}   \label{Sec_exp}
In this section, we numerically measure the performance of FedGDA-GT compared to Local SGDA with full gradients on a personal laptop by solving \eqref{global_emp_risk}. We consider first perform experiments on quadratic objective functions with $x$ and $y$ uncoupled. Then, we test our algorithm on the robust linear regression problem. In both cases, FedGDA-GT performs much better than Local SGDA with heterogeneous local objectives.

\subsection{Uncoupled quadratic objective functions}
We first consider the following local objective functions:
\begin{equation}    \label{exp_quadratic}
    f_i(x,y) = \frac{1}{2} x^T A_i^T A_i x - \frac{1}{2} y^T A_i^T A_i y + (A_i^T b_i)^T (2x - y), ~\forall i=1,\dots,m,
\end{equation}
where $x, y \in \mathbb{R}^d$ and $A_i \in \mathbb{R}^{n_i \times d}$ with $n_i$ representing the number of samples of agent $i$. We generate $A_i$, $b_i$ as follows:

For each agent, every entry of $A_i$, denoted by $[A_i]_{kl}$, is generated by Gaussian distribution $\mathcal{N}(0, (0.5i)^{-2})$. To construct $b_i$, we generate a random reference point $\theta_i \in \mathbb{R}^d$, where $\theta_i \sim \mathcal{N}(\mu_i, I_{d\times d})$. Each element of $\mu_i$ is drawn from $\mathcal{N}(\alpha, 1)$ with $\alpha \sim \mathcal{N}(0, 100)$. Then $b_i = A_i \theta_i + \epsilon_i$ with $\epsilon_i \sim \mathcal{N}(0, 0.25I_{n_i \times n_i})$. We set the dimension of model as $d = 50$ and number of samples as $n_i = 500$ and train the models with $m=20$ agents by Algorithm \ref{alg_FedGDA} and Algorithm \ref{alg_FedLC}, respectively. In order to compare them, the learning rate is $10^{-4}$ for both algorithms and we choose Local SGDA with $K=1$, which is equivalent to a centralized GDA, as the baseline.
\begin{figure}[h]
	\centering
	%\subfigure[$K=10$]{\includegraphics[width=6.8cm]{Figures/FedGDA-corr_vs_FedGDA_synthetic10.eps}}
	%\quad
	\subfigure[$K=20$]{\includegraphics[width=4.8cm]{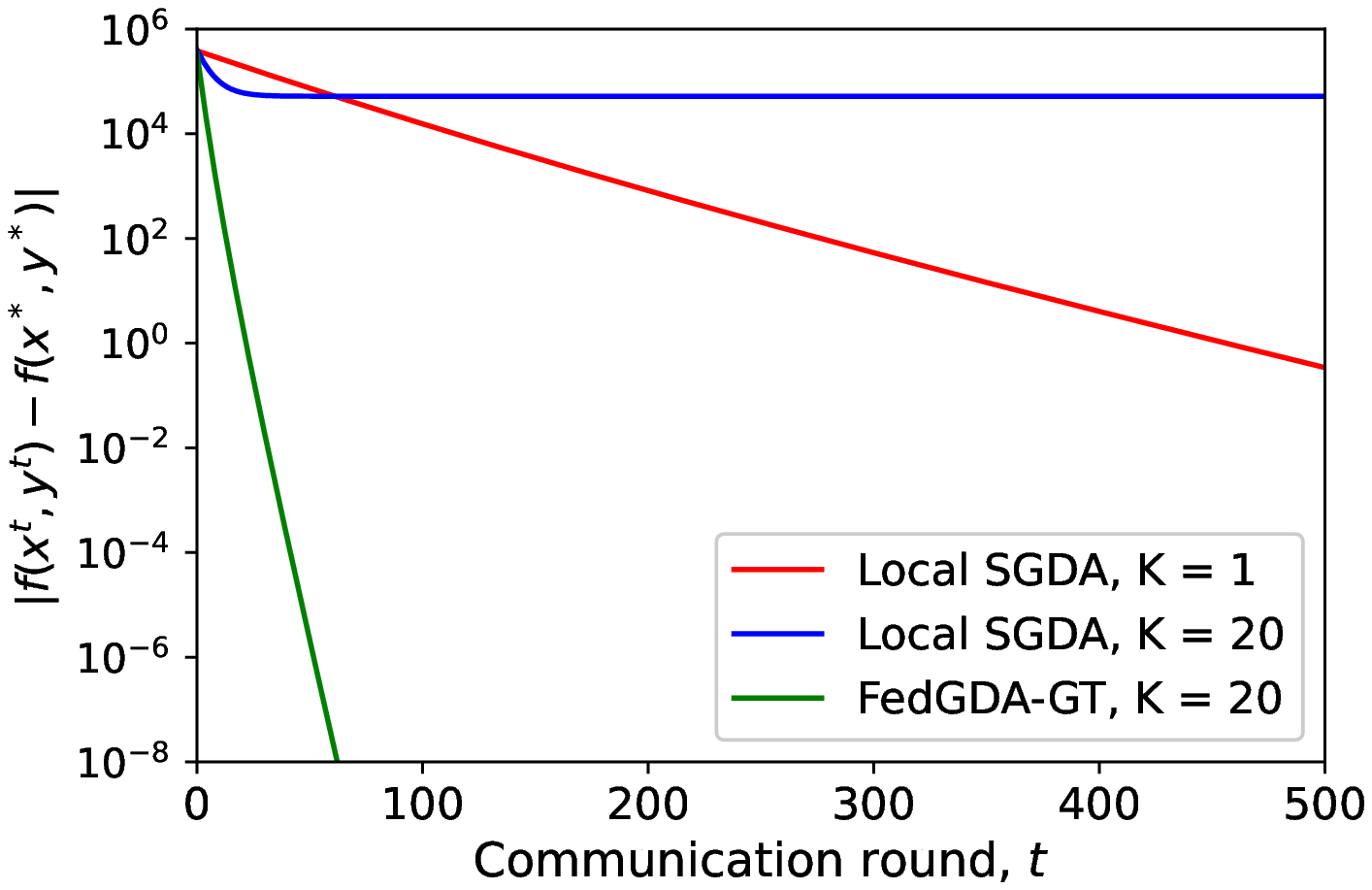}}
	\qquad
	\subfigure[$K=50$]{\includegraphics[width=4.8cm]{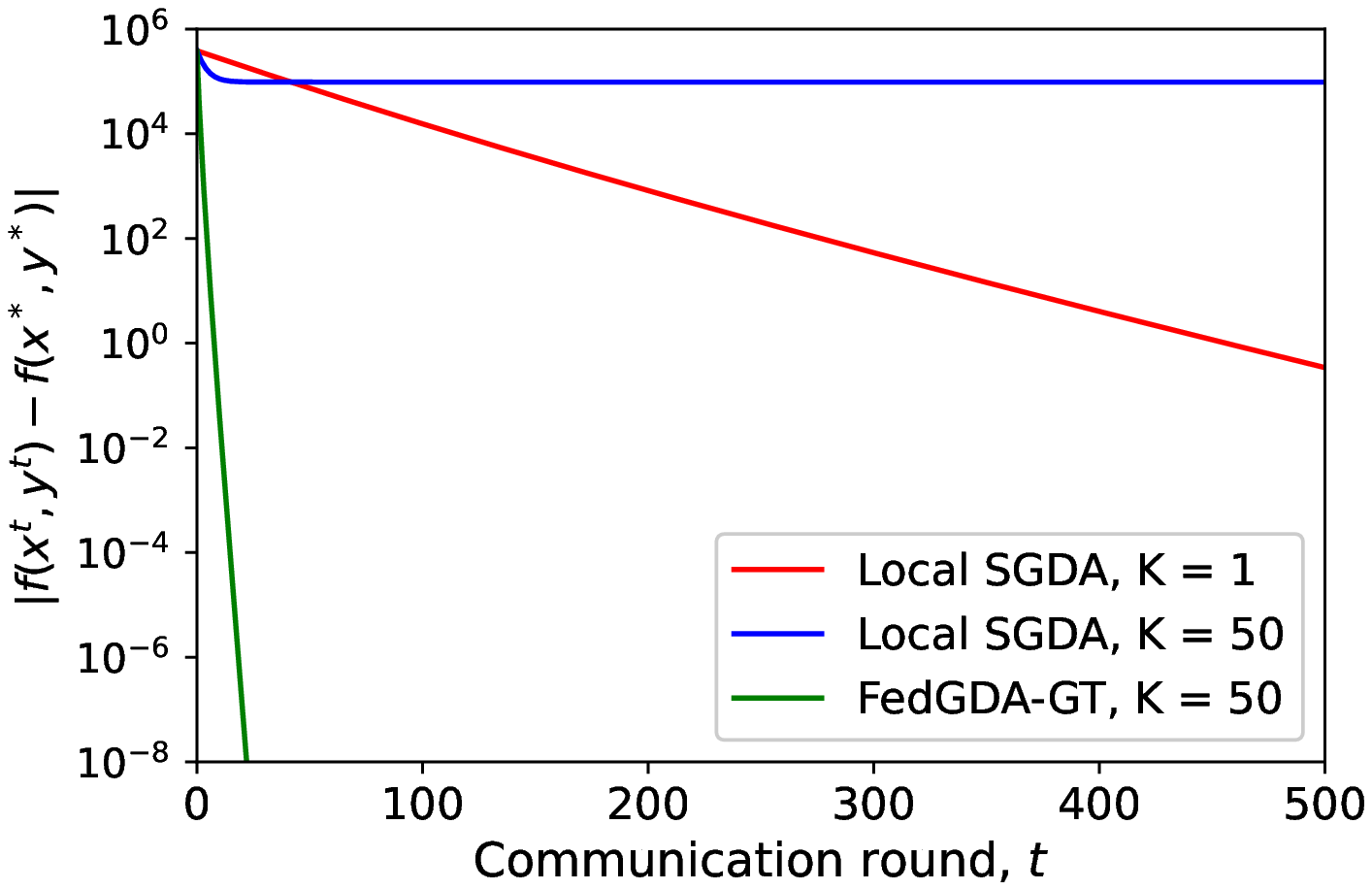}}
	\caption{Local SGDA and FedGDA-GT with constant stepsizes under different numbers of local updates}
	\label{fig_Quadratic_c=10}
\end{figure}
Figure \ref{fig_Quadratic_c=10} shows the trajectories of Algorithms \ref{alg_FedGDA} and \ref{alg_FedLC} under objective functions constructed by \eqref{exp_quadratic}, respectively. Different numbers of local updates are selected (with $K=20$ and $K=50$). In this heterogeneous setting, we can see that FedGDA-GT achieves linear convergence, converging to a more accurate solution with significantly fewer rounds of communication, compared with Local SGDA and centralized GDA. Moreover, our numerical results suggest that Local SGDA may converge to a non-optimal point (optimality gap over $10^4$), which conforms with our Theorem \ref{Thm_FedGDA-fixed-point}. %, cannot converge to a correct fixed point. the convergence error is quite large (even over $10^4$).%, which is intolerable in practice since such inaccurate model definitely suggests failure in training and testing. 

\subsection{Robust linear regression}
Next, we consider the problem of robust linear regression, which is widely studied in estimation with gross error \cite{RLR11,Attias19}. As the same formulation in \cite{Deng21}, each agent's loss function is defined by
\begin{equation}    \label{exp_RLR}
    f_i (x, y) = \frac{1}{n_i} \sum_{j=1}^{n_i} (x^T (a_{i,j} + y) - b_{i,j})^2 + \frac{1}{2} \Vert x \Vert^2, ~ \forall i = 1,\dots,m,
\end{equation}
where $(a_{i,j}, b_{i,j})$ is the $j$th data sample of agent $i$, $n_i$ is local sample size. Specifically, in \eqref{exp_RLR}, $x \in \mathbb{R}^d$ represents the model of linear regression and $y \in \mathbb{R}^d$ represents the gross noise aimed at contaminating each sample. We assume that there is an upper bound on the noise, i.e., $\Vert y \Vert \le 1$. By solving $\min_{x \in \mathbb{R}^d} \max_{\Vert y \Vert \le 1} \frac{1}{m}\sum_{i=1}^m f_i(x,y)$, we obtain a global robust model of the linear regression problem even under the worst contamination of gross noise. To measure the convergence of algorithms, we use the robust loss, i.e., given a model $\hat{x}$, the corresponding robust loss \cite{Deng21,Gauri22} is defined by $\tilde{f}(\hat{x})=\max_{\Vert y \Vert \le 1}\sum_{i=1}^m f_i(\hat x, y)$.
% %\begin{equation}
%     \tilde{f}(\hat{x}) = \max_{\Vert y \Vert \le 1} \frac{1}{m} \sum_{i=1}^m \sum_{j=1}^{n_i} \frac{1}{n_i}(\hat{x}^T (a_{i,j} + y) - b_{i,j})^2 + \frac{1}{2} \Vert \hat{x} \Vert^2 .  \nonumber
% \end{equation}
%Since solving the above maximization subproblem at each iteration is too computationally expensive, one alternative is to run several steps of gradient ascent to get an inexact solution of $\arg \max_{\Vert y \Vert \le 1}f(\hat{x})$.

We generate local models and data as follows: the local model $x_i^*$ is generated by a multivariate normal distribution. The output for agent $i$ is given by $b_{i,j} = (x_i^*)^T a_{i,j} + \epsilon_j$ with $\epsilon_j \sim \mathcal{N}(0,1)$. Each input point $a_{i,j}$ is with dimension $d$ and drawn from a Gaussian distribution $a_{i,j} \sim \mathcal{N}(\mu_i, K_i)$ where $\mu_i \sim \mathcal{N}(c_i, I_{d \times d})$ and $K_i = i^{-1.3}I_{d \times d}$. Each element of $c_i$ is drawn from $\mathcal{N}(0, \alpha^2)$. By choosing different $\alpha$, we control the heterogeneity of local data and hence $f_i(x,y)$.

In this experiment, we compare Algorithms \ref{alg_FedGDA} and \ref{alg_FedLC} under different heterogeneity levels, i.e., $\alpha = 1$, $\alpha=5$ and $\alpha = 20$. For each case, we choose the same constant $\eta$ for both Local SGDA and FedGDA-GT. As shown in Figure \ref{fig_RLR}, when local agents are more heterogeneous, FedGDA-GT performs better than Local SGDA, which lies not only in faster convergence but also smaller robust loss. Specifically, when $\alpha = 1$, two algorithms almost have the same performance. To explain this phenomenon, let us recall FedGDA-GT again. Smaller $\alpha$ essentially means more similar local objectives. In particular, $\alpha=0$ corresponds to i.i.d. cases. In this sense, the local updates of FedGDA-GT become the same as that in Local SGDA, which indicates similar performance as shown in Figure \ref{fig_RLR}(a).
\begin{figure}[h]
    \centering
    \subfigure[$\alpha=1$]
    {\includegraphics[width=4.6cm]{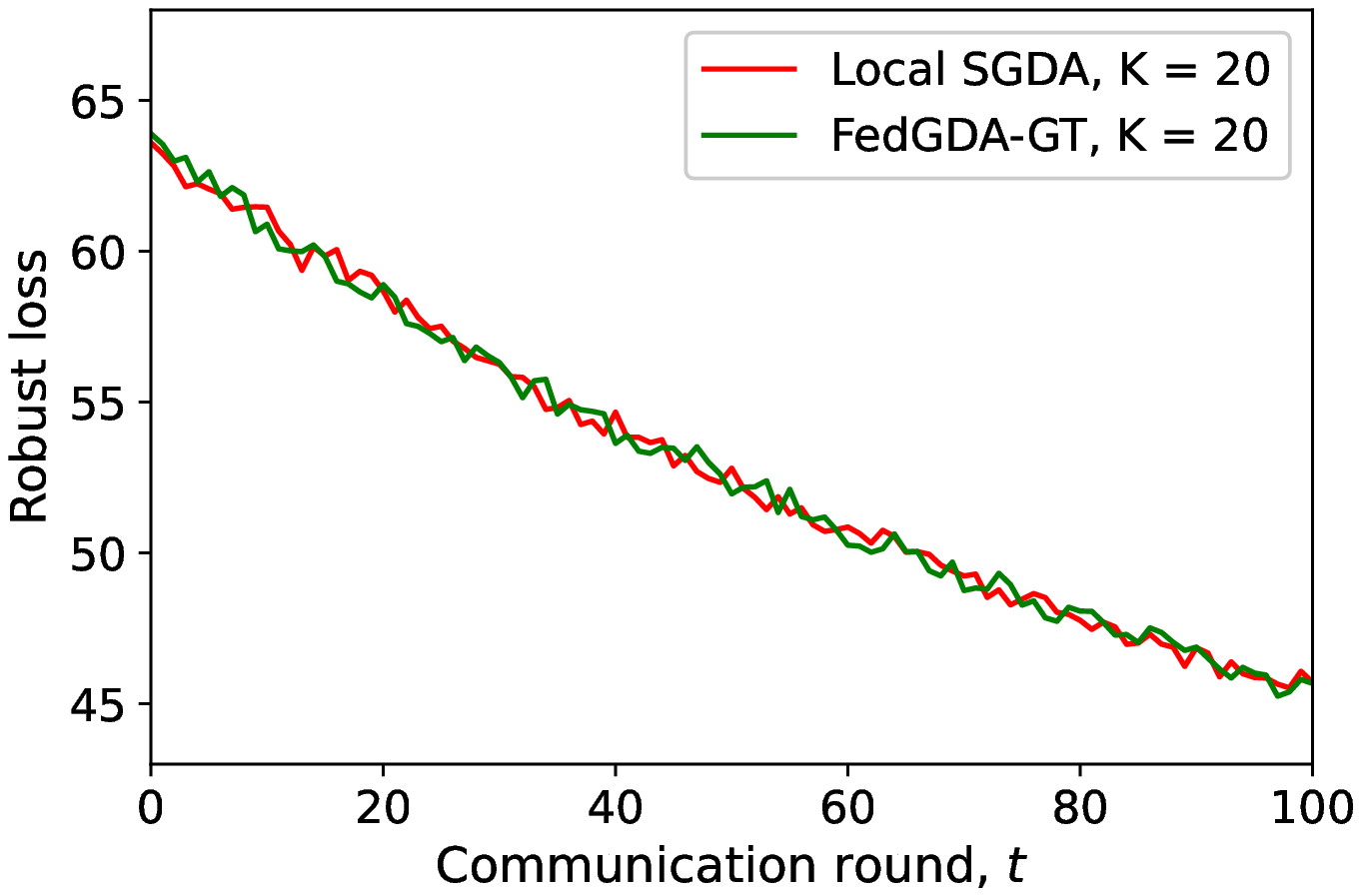}}
    %\quad
    \subfigure[$\alpha=5$]
    {\includegraphics[width=4.6cm]{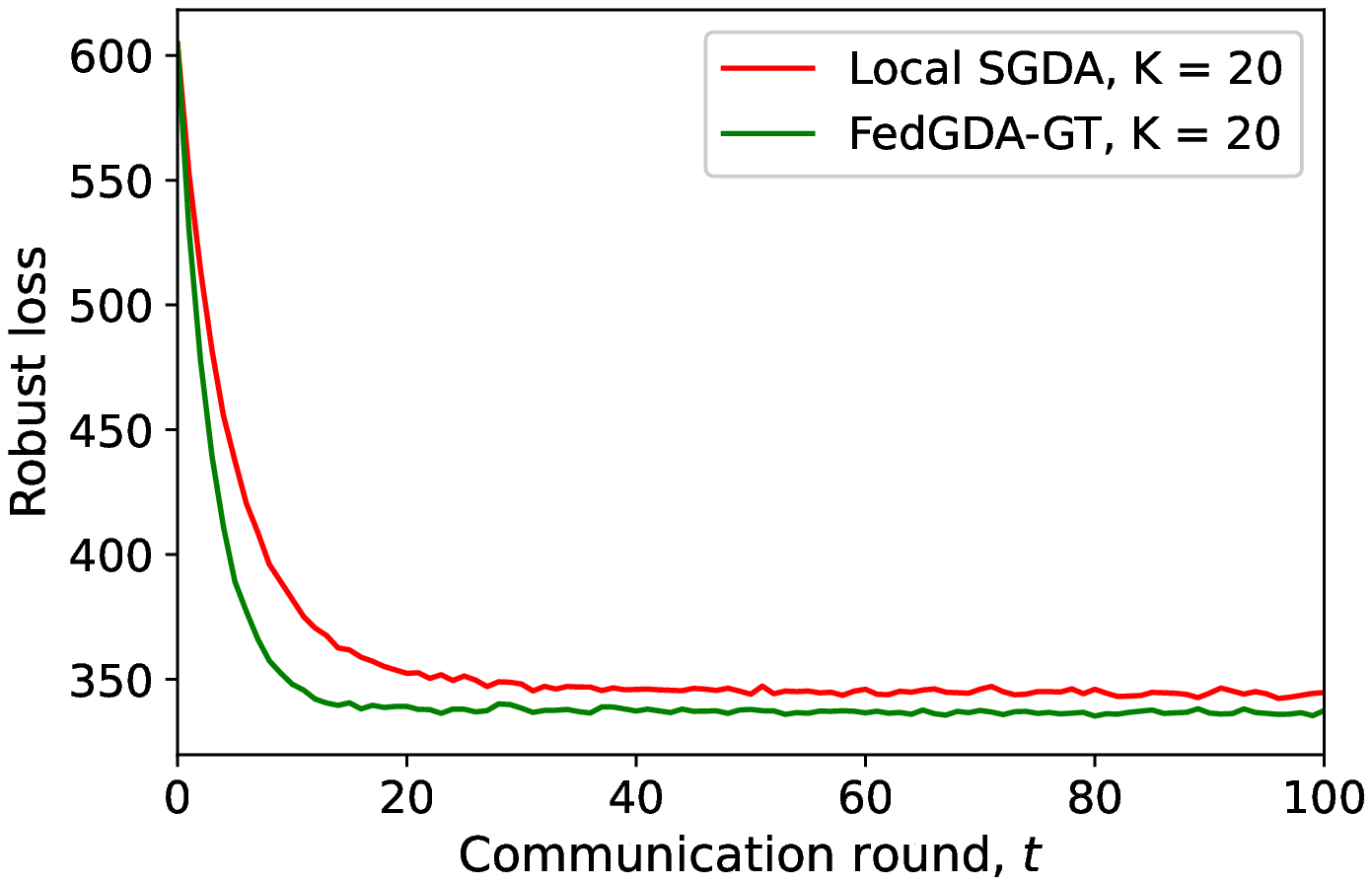}}
    %\quad
    %\subfigure[$\alpha=10$]
    %{\includegraphics[width=3.6cm]{Figures/FedGDA-corr_vs_FedGDA_RLR_alpha10.eps}}
    %\quad
    \subfigure[$\alpha=20$]
    {\includegraphics[width=4.6cm]{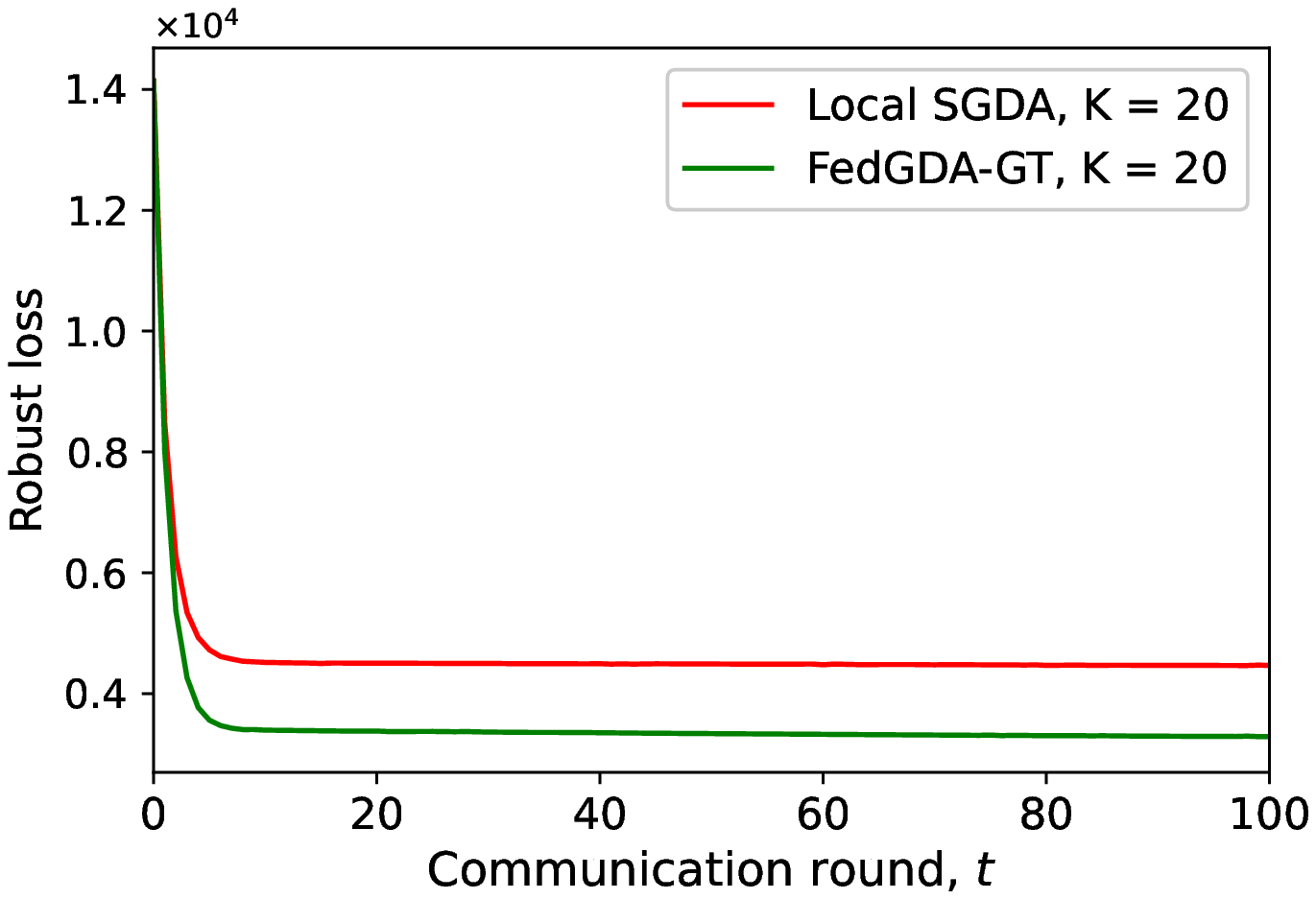}}
    \caption{Local SGDA and FedGDA-GT under different heterogeneity levels}
    \label{fig_RLR}
\end{figure}

%%%%%%%%%%%%%%%%%%%%%%%%%%%%%%%%%%%%%%%%%%%%%%%%%%%%%%%%%%%%%%%%%
%%%%%%%%%%%%%%%%%%%%%%%%%%%%%%%%%%%%%%%%%%%%%%%%%%%%%%%%%%%%%%%%%
%%%%%%%%%%%%%%%%%%%%%%%%%%%%%%%%%%%%%%%%%%%%%%%%%%%%%%%%%%%%%%%%%
%%%%%%%%%%%%%%%%%%%%%%%%%%%%%%%%%%%%%%%%%%%%%%%%%%%%%%%%%%%%%%%%%
\section{Conclusion}
In this paper, we investigate the federated minimax learning problem. We first characterize the fixed-point behavior of a recent algorithm Local SGDA to show that it presents a tradeoff between communication efficiency and model accuracy and cannot achieve linear convergence under constant learning rates. To resolve this issue, we propose FedGDA-GT that guarantees exact linear convergence and reaches $\epsilon$-optimality with $\mathcal{O}(\log (1/\epsilon))$ time, which is the same as centralized GDA method. Then, we study the generalization properties of  distributed minimax learning problems. We establish generalization error bounds without strong assumptions on local distributions and loss functions based on Rademacher complexity. The bounds match existing results of centralized minimax learning problems. %Moreover, FedGDA-GT places no limitation on the level of objective heterogeneity of agents.
Finally, we compare FedGDA-GT with two state-of-the-art algorithms, Local SGDA and GDA, through numerical experiments and show that FedGDA-GT outperforms in efficiency and/or accuracy.

\begin{ack}
%Use unnumbered first level headings for the acknowledgments. All acknowledgments
%go at the end of the paper before the list of references. Moreover, you are required to declare
%funding (financial activities supporting the submitted work) and competing interests (related financial activities outside the submitted work).
%More information about this disclosure can be found at: \url{https://neurips.cc/Conferences/2022/PaperInformation/FundingDisclosure}.

This work was supported by the NSF NRI 2024774.

%Do {\bf not} include this section in the anonymized submission, only in the final paper. You can use the \texttt{ack} environment provided in the style file to autmoatically hide this section in the anonymized submission.
\end{ack}

%%%%%%%%%%%%%%%%%%%%%%%%%%%%%%%%%%%%%%%%%%%%%%%%%%%%%%%%%%

%%%%%%%%%%%%%%%%%%%%%%%%%%%%%%%%%%%%%%%%%%%%%%%%%%%%%%%%%%%%

\clearpage
\section*{Checklist}

\begin{enumerate}

\item For all authors...
\begin{enumerate}
  \item Do the main claims made in the abstract and introduction accurately reflect the paper's contributions and scope?
    \answerYes{}
  \item Did you describe the limitations of your work?
    \answerYes{}
  \item Did you discuss any potential negative societal impacts of your work?
    \answerNA{}
  \item Have you read the ethics review guidelines and ensured that your paper conforms to them?
    \answerYes{}
\end{enumerate}

\item If you are including theoretical results...
\begin{enumerate}
  \item Did you state the full set of assumptions of all theoretical results?
    \answerYes{}
        \item Did you include complete proofs of all theoretical results?
    \answerYes{}
\end{enumerate}

\item If you ran experiments...
\begin{enumerate}
  \item Did you include the code, data, and instructions needed to reproduce the main experimental results (either in the supplemental material or as a URL)?
    \answerYes{}
  \item Did you specify all the training details (e.g., data splits, hyperparameters, how they were chosen)?
    \answerYes{}
        \item Did you report error bars (e.g., with respect to the random seed after running experiments multiple times)?
    \answerNA{}
        \item Did you include the total amount of compute and the type of resources used (e.g., type of GPUs, internal cluster, or cloud provider)?
    \answerYes{}
\end{enumerate}

\item If you are using existing assets (e.g., code, data, models) or curating/releasing new assets...
\begin{enumerate}
  \item If your work uses existing assets, did you cite the creators?
    \answerNA{}
  \item Did you mention the license of the assets?
    \answerNA{}
  \item Did you include any new assets either in the supplemental material or as a URL?
    \answerNA{}
  \item Did you discuss whether and how consent was obtained from people whose data you're using/curating?
    \answerNA{}
  \item Did you discuss whether the data you are using/curating contains personally identifiable information or offensive content?
    \answerNA{}
\end{enumerate}

\item If you used crowdsourcing or conducted research with human subjects...
\begin{enumerate}
  \item Did you include the full text of instructions given to participants and screenshots, if applicable?
    \answerNA{}
  \item Did you describe any potential participant risks, with links to Institutional Review Board (IRB) approvals, if applicable?
    \answerNA{}
  \item Did you include the estimated hourly wage paid to participants and the total amount spent on participant compensation?
    \answerNA{}
\end{enumerate}

\end{enumerate}

%%%%%%%%%%%%%%%%%%%%%%%%%%%%%%%%%%%%%%%%%%%%%%%%%%%%%%%%%%%%
%%%%%%%%%%%%%%%%%%%%%%%%%%%%%%%%%%%%%%%%%%%%%%%%%%%%%%%%%%%%
%%%%%%%%%%%%%%%%%%%%%%%%%%%%%%%%%%%%%%%%%%%%%%%%%%%%%%%%%%%%
%%%%%%%%%%%%%%%%%%%%%%%%%%%%%%%%%%%%%%%%%%%%%%%%%%%%%%%%%%%%
%%%%%%%%%%%%%%%%%%%%%%%%%%%%%%%%%%%%%%%%%%%%%%%%%%%%%%%%%%%%
%%%%%%%%%%%%%%%%%%%%%%%%%%%%%%%%%%%%%%%%%%%%%%%%%%%%%%%%%%%%
\clearpage

\appendix

\section{Applications of distributed/federated minimax problems}    \label{Apx_applications}
In this section, we consider specific instantiations of \eqref{global_emp_risk} and \eqref{global_pop_minimax}. Two examples are presented: one is federated generative adversarial networks, another is agnostic federated learning. We show that both of them are special cases of the general framework considered in the paper. 

\subsection{Federated generative adversarial networks}
In \cite{FedGAN}, the authors consider to train GANs in a federated way, where $m$ agents with corresponding local datasets cooperate to learn a common model which is essentially the model of centralized GAN. Then, for each agent, its local objective function is defined by
\begin{equation}
    R_i(x,y) = -\mathbb{E}_{\xi \sim P_i}\left[ \log \phi_{y}(\xi) \right] - \mathbb{E}_{\xi' \sim Q_i(x)}\left[ \log (1 - \phi_{y}(\xi')) \right] \nonumber
\end{equation}
where $\phi_{y}$ is the discriminator and $Q_i(x)$ is the distribution to generate fake data of the generator. And the objective of the centralized GAN is given by
\begin{equation}
    R(x,y) = \frac{1}{m} \sum_{i=1}^m R_i(x,y)  \nonumber
\end{equation}
when identical sample sizes are assumed. This is essentially the same as our formulation.

\subsection{Agnostic federated learning}
The framework of agnostic federated learning was first proposed and analyzed in \cite{Mohri19}. where the centralized model is leaned for any possible target distribution that is formed by a convex combination of all agents' local distributions. In particular, let $P_i$ denote the distribution of agent $i$. Then, the target distribution is formed by $\sum_{i=1}^m \lambda_i^* P_i$ for some unknown $\lambda^*$ such that $\lambda^* \in \Lambda$, where $\Lambda$ represents a simplex. Then, agnostic federated learning is aimed at learning a model $\theta^*$ that performs best under the worst case, i.e.,
\begin{equation}
    \theta^* = \arg\min_{\theta \in \Theta} \left\{R(\theta, \Lambda) := \max_{\lambda \in \Lambda}\sum_{i=1}^m \lambda_i R_i(\theta) \right\}  \nonumber
\end{equation}
where $R_i(\theta) = \mathbb{E}_{\xi \sim P_i}[l(\theta;\xi)]$ is the local population risk. The empirical version of the problem can be derived similarly. Note that this formulation is essentially included by our problem.

%%%%%%%%%%%%%%%%%%%%%%%%%%%%%%%%%%%%%%%%%%%%%%%%%%%%%%%%%%%%%%%
%%%%%%%%%%%%%%%%%%%%%%%%%%%%%%%%%%%%%%%%%%%%%%%%%%%%%%%%%%%%%%%
%%%%%%%%%%%%%%%%%%%%%%%%%%%%%%%%%%%%%%%%%%%%%%%%%%%%%%%%%%%%%%%
%%%%%%%%%%%%%%%%%%%%%%%%%%%%%%%%%%%%%%%%%%%%%%%%%%%%%%%%%%%%%%%
\section{Proof of Proposition \ref{Thm_FedGDA-fixed-point}}
%\begin{proof}
    Given \eqref{FedGDA_operator}, it is straightforward that 
    \begin{eqnarray}
        x^{t+1}_{i,K} &=& \mathcal{D}_i^K (x^{t+1}_{i,0}, y^{t+1}_{i,0}) ,    \nonumber   \\
        y^{t+1}_{i,K} &=& \mathcal{A}_i^K (x^{t+1}_{i,0}, y^{t+1}_{i,0}) .    \nonumber
    \end{eqnarray}
    Noting $(x^{t+1}, y^{t+1}) = \frac{1}{m}\sum_{i=1}^m (x^{t+1}_{i,K}, y^{t+1}_{i,K})$ and $(x^t, y^t) = (x^{t+1}_{i,0}, y^{t+1}_{i,0})$,
    \begin{eqnarray}
        x^{t+1} &=& \frac{1}{m}\sum_{i=1}^m \mathcal{D}_i^K (x^t, y^t) ,    \nonumber   \\
        y^{t+1} &=& \frac{1}{m}\sum_{i=1}^m \mathcal{A}_i^K (x^t, y^t) .    \nonumber
    \end{eqnarray}
    Further using $\lim_{t\to\infty} (x^t, y^t) = (x^*, y^*)$ gives
    \begin{eqnarray}
        \frac{1}{m}\sum_{i=1}^m \sum_{k=0}^{K-1} \nabla_{x} f_i(\mathcal{D}_i^k(x^*, y^*), \mathcal{A}_i^k(x^*, y^*)) &=& 0 ,    \nonumber   \\
        \frac{1}{m}\sum_{i=1}^m \sum_{k=0}^{K-1} \nabla_{y} f_i(\mathcal{D}_i^k(x^*, y^*), \mathcal{A}_i^k(x^*, y^*)) &=& 0 ,    \nonumber
    \end{eqnarray}
    which completes the proof.
%\end{proof}

%%%%%%%%%%%%%%%%%%%%%%%%%%%%%%%%%%%%%%%%%%%%%%%%%%%%%%%%%%%%%%%
%%%%%%%%%%%%%%%%%%%%%%%%%%%%%%%%%%%%%%%%%%%%%%%%%%%%%%%%%%%%%%%
%%%%%%%%%%%%%%%%%%%%%%%%%%%%%%%%%%%%%%%%%%%%%%%%%%%%%%%%%%%%%%%
%%%%%%%%%%%%%%%%%%%%%%%%%%%%%%%%%%%%%%%%%%%%%%%%%%%%%%%%%%%%%%%

\section{An illustrative example for Local SGDA with constant stepsizes}    \label{Apx_example}
We illustrate the inexact convergence of local SGDA with constant stepzie through a simple instance of \eqref{global_emp_risk} where only two agents cooperate to find a minimax point of $f(x,y)$ by Local SGDA using full gradient information. We further assume that each $f_i(x,y)$ is strongly-convex-strongly-concave and Lipschitz smooth such that the minimax point is unique and linear convergence is possible to reach. Specifically, we construct local objectives as follows:
\begin{eqnarray}
    f_1 (x,y) &=& x^2 - y^2 - (x - y)  \nonumber \\
    f_2 (x,y) &=& 4x^2 - 4y^2 - 32(x - y)   \nonumber   
\end{eqnarray}
where the minimax point is $x^* = y^* = \left( \sum_{i=1}^2 2i^2 \right)^{-1} \sum_{i=1}^2 (31i - 30)$.
By Proposition \ref{Thm_FedGDA-fixed-point}, a straightforward calculation gives
\begin{eqnarray}
    x^*_{\mathrm{Local-SGDA}} &=& \left(\sum_{i=1}^2 \sum_{k=0}^{K-1} 2i^2(1 - 2\eta_x i^2)^k \right)^{-1} \sum_{i=1}^2 \sum_{k=0}^{K-1} (31i - 30)(1 - 2\eta_x i^2)^k ,   \nonumber    \\
    y^*_{\mathrm{Local-SGDA}} &=& \left(\sum_{i=1}^2 \sum_{k=0}^{K-1} 2i^2(1 - 2\eta_y i^2)^k \right)^{-1} \sum_{i=1}^2 \sum_{k=0}^{K-1} (31i - 30)(1 - 2\eta_y i^2)^k .  \nonumber
\end{eqnarray}
In general $x^*_{\mathrm{FedGDA}} \ne x^*$, $y^*_{\mathrm{FedGDA}} \ne y^*$ when $K \ge 2$. Therefore, we see that Local SGDA has incorrect fixed points when constant stepsizes are used even under deterministic scenarios. 

In the sequel, we empirically show the effect of different numbers of local updates on the fixed point. We consider cases with $K = 1$, $K = 10$, $K = 20$, $K = 50$. The stepsizes $\eta_x$ and $\eta_y$ are set by $\eta_x = \eta_y = 0.1$ when $K=1$ and by $\eta_x = \eta_y = 0.001$ for the remaining cases. The initial points $(x^0, y^0)$ for four cases are identical for the convenience of comparison. From Figure \ref{fig_FedGDA-fails}, when $K=1$ Local SGDA reduces to centralized GDA and converges to the minimax point $(x^*, y^*)$ linearly by strong-convexity-strong-concavity and Lipschitz smoothness assumptions. However, for $K=10,20,50$, given identical stepsizes, larger the number of local updates is, fewer communication rounds are needed until convergence, but farther the limit points are from the optimal one. Another point that is worthy to note is that convergence error between the minimax point and the fixed point of Local SGDA can be too large to be neglected (even over $10^3$ in Section \ref{Sec_exp}), although the errors shown in Figure \ref{fig_FedGDA-fails} are relatively small.

\begin{figure}[h]
	\begin{center}
		\includegraphics[width=10cm]{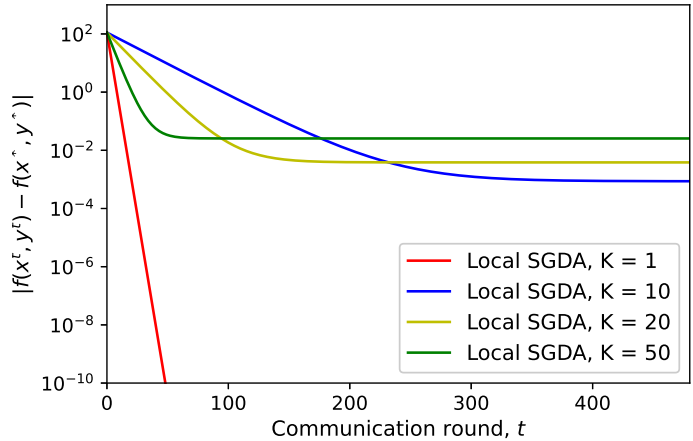}
		\caption{Local SGDA with constant stepsizes under different numbers of local updates}
		\label{fig_FedGDA-fails}
	\end{center}
\end{figure}

%%%%%%%%%%%%%%%%%%%%%%%%%%%%%%%%%%%%%%%%%%%%%%%%%%%%%%%%%%%%%%%
%%%%%%%%%%%%%%%%%%%%%%%%%%%%%%%%%%%%%%%%%%%%%%%%%%%%%%%%%%%%%%%
%%%%%%%%%%%%%%%%%%%%%%%%%%%%%%%%%%%%%%%%%%%%%%%%%%%%%%%%%%%%%%%
%%%%%%%%%%%%%%%%%%%%%%%%%%%%%%%%%%%%%%%%%%%%%%%%%%%%%%%%%%%%%%%
\section{Convergence analysis of FedGDA-GT} \label{Apx_convergence}

\subsection{Proof of Lemma \ref{Lmm_unique-point}}
First, we introduce the definition of saddle point of $f(x,y)$:
\begin{definition}  \label{def_saddle-point}
    The point $(x^*, y^*)$ is said to be a saddle point of $f(x,y)$ if 
    $$ f(x^*, y) \le f(x^*, y^*) \le f(x, y^*), \forall x \in \mathcal{X}, y \in \mathcal{Y} .$$
\end{definition}
Obviously, by Definitions \ref{def_minimax-point} and \ref{def_saddle-point}, we know that any saddle point of $f(x,y)$ is also a minimax point of $f(x,y)$. Then, any saddle point in the interior of $\mathcal{X}\times \mathcal{Y}$ must satisfy Lemma \ref{Lmm_stationary-point}. Moreover, when $f(x,y)$ is strongly-convex-strongly-concave, we show that any minimax point is also a saddle point, stated as follows:
\begin{lemma}   \label{Lmm_minimax-is-saddle}
    Suppose $f(x,y)$ satisfies Assumptions \ref{assp_SCSC} and \ref{assp_convex-sets}. If $(x^*,y^*)$ is a minimax point of $f(x,y)$, then it is also a saddle point of $f(x,y)$. Moreover, 
    $$
        \nabla_x f(x^*,y^*) = \nabla_y f(x^*, y^*) = 0 ~~ \Longleftrightarrow ~~ (x^*, y^*)~\mathrm{~ is~ a~ saddle ~point ~ of}~ f(x,y) .
    $$
\end{lemma}
\begin{proof}
    By Assumption \ref{assp_SCSC}, we have
    \begin{equation}
        \nabla_{xx}^2 f(x,y) \succeq \mu I , \quad \nabla_{yy}^2 f(x,y) \preceq -\mu I, ~~~\forall x,y  .    \nonumber
    \end{equation}
    By Lemma \ref{Lmm_stationary-point} and Proposition 5 in \cite{Jin-minimax}, we obtain that under Assumptions \ref{assp_SCSC} and \ref{assp_convex-sets},
    $$
        \nabla_x f(x^*,y^*) = \nabla_y f(x^*, y^*) = 0 ~~ \Longleftrightarrow ~~ (x^*, y^*)~\mathrm{~ is~ a~ saddle ~point ~ of}~ f(x,y) .
    $$
    Noting that $(x^*, y^*)$ is a minimax point implies $\nabla_x f(x^*,y^*) = \nabla_y f(x^*, y^*) = 0$ by Lemma \ref{Lmm_stationary-point}, this completes the proof.
\end{proof}
Next, we provide the uniqueness statement of saddle point $(x^*, y^*)$.
\begin{lemma}   \label{Lmm_unique-saddle}
    Under Assumption \ref{assp_SCSC}, the saddle point $(x^*, y^*)$ of $f(x,y)$ is unique in $\mathcal{X} \times \mathcal{Y}$.
\end{lemma}
\begin{proof}
    By Assumption \ref{assp_SCSC}, it yields given any $y \in \mathbb{R}^q$ and $\alpha \in (0,1)$,
    \begin{equation}    \label{eq_apx-SCSC-coro}
        f(\alpha x + (1-\alpha)z , y) \le \alpha f(x,y) + (1-\alpha)f(z,y) - \frac{\mu}{2}\alpha(1-\alpha)\Vert z - x \Vert^2 , ~~ \forall x,z.
    \end{equation}
    Suppose there exists some saddle point $(u^*, v^*) \ne (x^*, y^*)$. Then $f(x^*, y^*) = f(u^*, v^*)$ must hold. Otherwise without loss of generality, assuming $f(x^*, y^*) < f(u^*, v^*)$, by the definition of saddle points, the fact $f(x^*, v^*) \le f(x^*, y^*) < f(u^*, v^*)$ contradicts $f(u^*, v^*) \le f(x^*, v^*)$.
    
    Then, by \eqref{eq_apx-SCSC-coro} and Definition \ref{def_saddle-point},
    \begin{eqnarray}
        f(x^*, y^*) \le f(\alpha x^* + (1-\alpha)u^*, y^*) &<& \alpha f(x^*, y^*) + (1-\alpha)f(u^*, y^*)     \nonumber   \\
        &\le& \alpha f(x^*, y^*) + (1-\alpha)f(u^*, v^*) ,    \nonumber
    \end{eqnarray}
    which implies $f(x^*, y^*) < f(u^*, v^*)$, contradicting $f(x^*, y^*) = f(u^*, v^*)$. This completes the proof.
\end{proof}

Finally, combining Lemmas \ref{Lmm_minimax-is-saddle} and \ref{Lmm_unique-saddle} gives Lemma \ref{Lmm_unique-point}.

\subsection{Technical Lemmas}
Before the convergence proof of Theorem \ref{Thm_convergence}, we need several technical lemmas.
\begin{lemma}
    \textbf{(Relaxed triangle inequality)} Let $v_1, \dots, v_n$ be $n$ vectors in $\mathbb{R}^d$. Then,
    $$
        \left\Vert \sum_{i=1}^n v_i \right\Vert^2 \le n \sum_{i=1}^n \Vert v_i \Vert^2.
    $$
\end{lemma}

\begin{lemma}   \label{lmm_SM}
    Let $F_i(z) = ( \nabla_{x}f_i(x,y), -\nabla_{y}f_i(x,y))$ where $z = (x,y)$. Under Assumption \ref{assp_SCSC}, $F_i(\cdot)$ is $\mu$-strongly monotone, $\forall i=1,\dots,m$, which means
    $$
        \langle F_i(z) - F_i(z'), z - z' \rangle \ge \mu \Vert z-z' \Vert^2,~ \forall z,z' \in \mathbb{R}^{p+q}.
    $$
\end{lemma}

\begin{proof}

    Let $g_i(x,y) = f_i(x,y) - \frac{\mu}{2}\Vert x \Vert^2 + \frac{\mu}{2}\Vert y \Vert^2$ and $G_i(z) = (\nabla_x g_i(x,y), -\nabla_y g_i(x,y))$, where $z=(x,y)$. From Assumption \ref{assp_SCSC}, it is obvious that $g_i(x,y)$ is convex-concave. Then, $F_i(z)$ is $\mu$-strongly monotone is equivalent to 
    $$
        \langle G_i(z) - G_i(z'), z - z' \rangle \ge 0, ~ \forall z, z' \in \mathbb{R}^{p+q}.
    $$
    
    Given the convex-concave property of $g_i(x,y)$, we have for any $z_1=(x_1, y_1)$, $z_2 = (x_2, y_2)$,
    \begin{eqnarray}
        g_i(x_2, y_1) &\ge& g_i(x_1, y_1) + \langle \nabla_x g_i(x_1, y_1), x_2 - x_1 \rangle,     \nonumber   \\
        -g_i(x_1, y_2) &\ge& -g_i(x_1, y_1) - \langle \nabla_y g_i(x_1, y_1), y_2 - y_1 \rangle,    \nonumber   \\
        g_i(x_1, y_2) &\ge& g_i(x_2, y_2) + \langle \nabla_x g_i(x_2, y_2), x_1 - x_2 \rangle,    \nonumber   \\
        -g_i(x_2, y_1) &\ge& -g_i(x_2, y_2) - \langle \nabla_y g_i(x_2, y_2), y_1 - y_2 \rangle.    \nonumber
    \end{eqnarray}
    Adding these four inequalities gives
    $$
        \langle \nabla_x g_i(x_1, y_1) - \nabla_x g_i(x_2, y_2), x_1 - x_2 \rangle - \langle g_i(x_1, y_1) - g_i(x_2, y_2), y_1 - y_2 \rangle \ge 0,
    $$
    which essentially indicates $\langle G_i(z_1) - G_i(z_2), z_1 - z_2 \rangle \ge 0$. This completes the proof.

    % We only need to prove $\nabla F_i(z) \succeq \mu I$, which is equivalent to $F_i(z)$ is $\mu$-strongly monotone.  By Assumption \ref{assp_SCSC}, we have
    % $$
    %     \nabla_{xx}^2 f_i(x,y) \succeq \mu I, ~~ \nabla_{yy}^2 f_i(x,y) \preceq -\mu I, ~~ \forall x,y.
    % $$
    % For any vector $w = (u, v) \in \mathbb{R}^{p+q}$, a simple calculation gives
    % \begin{eqnarray}
    %     w^T \nabla F_i(z) w &=&
    %     [u^T ~ v^T]\left[
    %     \begin{array}{cc}
    %         \nabla_{xx}^2 f_i(x,y) & \nabla_{xy}^2 f_i(x,y) \\
    %         -\nabla_{yx}^2 f_i(x,y) & -\nabla_{yy}^2 f_i(x,y)
    %     \end{array}
    %     \right] \left[\begin{array}{c}
    %          u  \\
    %          v 
    %     \end{array} \right] \nonumber   \\
    %   % &=& u^T \nabla_{xx}^2 f_i(x,y) u - v^T \nabla_{yy}^2 f_i(x,y) v + u^T \nabla_{xy}^2 v - v^T \nabla_{yx}^2 f_i(x,y) u    \nonumber   \\
    %     &=& u^T \nabla_{xx}^2 f_i(x,y) u - v^T \nabla_{yy}^2 f_i(x,y) v + u^T (\nabla_{xy}^2 f_i(x,y) - (\nabla_{yx}^2 f_i(x,y))^T) v   \nonumber   \\
    %     &\ge& \mu (\Vert u \Vert^2 + \Vert v \Vert^2) = \mu \Vert w \Vert^2 \nonumber
    % \end{eqnarray}
    % where $\nabla_{xy}^2 f_i(x,y) = (\nabla_{yx}^2 f_i(x,y))^T$ follows from the twice-continuous differentiability of $f_i(\cdot)$.
\end{proof}

%\begin{assumption}  \label{assp_nondeg-hess}
%    Given a saddle point $(x^*, y^*)$ of $f(x,y)$, $\nabla_{xx}^2 f(x^*, y^*)$ and $\nabla_{yy}^2 f(x^*, y^*)$ are non-degenerate.
%\end{assumption}

\begin{lemma}   \label{Lmm_contraction}
    For any $\mu$-strongly monotone and $L$-Lipschitz continuous operator $F(\cdot)$, there exists some $\lambda \in (0,1)$ such that given any $\eta \in (0, 2\mu / L^2)$,
    \begin{equation}
        \Vert u - v - \eta(F(u) - F(v)) \Vert \le \lambda \Vert u-v \Vert   \nonumber, ~\forall u,v .
    \end{equation}
\end{lemma}
\begin{proof}
    \begin{eqnarray}
        \Vert u-v-\eta(F(u)-F(v)) \Vert^2 &=& \Vert u-v \Vert^2 + \eta^2 \Vert F(u)-F(v) \Vert^2 - 2\eta \langle F(u)-F(v), u-v \rangle \nonumber   \\
        &\le& \Vert u-v \Vert^2 + \eta^2 L^2 \Vert u-v \Vert^2 - 2\eta \mu \Vert u-v \Vert^2    \nonumber   \\
        &=& (1 - \eta(2\mu - \eta L^2))\Vert u-v \Vert^2
    \end{eqnarray}
    where Lemma \ref{lmm_SM} is used.
    
    By setting $\lambda = 1 - \eta(2\mu - \eta L^2)$, we obtain $\lambda \in (0,1)$, $\forall \eta \in (0, 2\mu / L^2)$, which completes the proof.
\end{proof}

\subsection{Proof of Theorem \ref{Thm_convergence}}
In this section, we formally prove Theorem \ref{Thm_convergence}.

%\begin{proof}
    Define $z=(x,y)$, $F_i(z) = \left( \nabla_{x}f_i(x,y), -\nabla_{y}f_i(x,y) \right)$, $F(z) = \left( \nabla_{x}f(x,y), -\nabla_{y}f(x,y) \right)$. By definition, $F(z) = \frac{1}{m}\sum_{i=1}^m F_i(z)$. Denote $\mathrm{Proj}_{\mathcal{Z}}(\cdot) = \mathrm{Proj}_{\mathcal{X}\times \mathcal{Y}}(\cdot)$.
    
We focus on the updates within one outer iteration $t$ and may selectively drop the superscript $t$ in the following analysis for notational convenience. Then according to Algorithm \ref{alg_FedLC}, we obtain
    \begin{eqnarray}    
        z_{i,K} &=& z_{i,0} - \eta \sum_{k=0}^{K-1}\left( F_i(z_{i,k}) - F_i(z^t) + F(z^t) \right)    \nonumber   \\
        &=& z_{i,0} - \eta \sum_{k=0}^{K-1}F_i(z_{i,k}) + \eta K (F_i(z^t) - F(z^t)) . \nonumber
    \end{eqnarray}
    Note that $z^t = \frac{1}{m}\sum_{i=1}^m z_{i,0}$ and $z^{t+1} = \mathrm{Proj}_{\mathcal{Z}}\left(\frac{1}{m}\sum_{i=1}^m z_{i,K}\right)$, it yields
    \begin{eqnarray}    \label{proof_zt_iter}
        z^{t+1} &=& \mathrm{Proj}_{\mathcal{Z}} \left( \frac{1}{m}\sum_{i=1}^m z_{i,0} - \frac{\eta}{m}\sum_{i=1}^{m}\sum_{k=0}^{K-1}F_i(z_{i,k}) + \frac{\eta K}{m}\sum_{i=1}^{m}(F_i(z^t) - F(z^t)) \right)   \nonumber   \\
        &=& \mathrm{Proj}_{\mathcal{Z}} \left( z^t - \frac{\eta}{m}\sum_{i=1}^m \sum_{k=0}^{K-1}F_i(z_{i,k}) \right).
    \end{eqnarray}
    
    Then, we have
    \begin{eqnarray}    \label{proof_err}
        \Vert z^{t+1} - z^* \Vert^2 &\le& \left\Vert z^t - z^* - \frac{\eta}{m}\sum_{i=1}^m \sum_{k=0}^{K-1}F_i(z_{i,k}) \right\Vert^2 \nonumber   \\
        &=& \Vert z^t - z^* \Vert^2 + \underbrace{\left\Vert \frac{\eta}{m}\sum_{i=1}^m \sum_{k=0}^{K-1}F_i(z_{i,k}) \right\Vert^2}_{\tau_1}     \nonumber   \\
        && \underbrace{- 2\eta \langle \frac{1}{m}\sum_{i=1}^m \sum_{k=0}^{K-1}F_i(z_{i,k}), z^t - z^* \rangle}_{\tau_2}
    \end{eqnarray}
    where we use the fact that $\left\Vert \mathrm{Proj}_{\mathcal{Z}}(z_1) - \mathrm{Proj}_{\mathcal{Z}}(z_2) \right\Vert \le \left\Vert z_1 - z_2 \right\Vert$ and $z^* = (x^*, y^*)$.
    
    Next, we will bound $\tau_1$. By noting $F(z^*) = 0$, we have
    \begin{eqnarray}    \label{proof_bound-incrm-err}
        \tau_1 &=& \frac{\eta^2}{m^2}\left\Vert \sum_{i=1}^m \sum_{k=0}^{K-1}F_i(z_{i,k}) \right\Vert^2 \nonumber   \\
        &=& \frac{\eta^2}{m^2}\left\Vert \sum_{i=1}^m \sum_{k=0}^{K-1} \left(F_i(z_{i,k}) - F_i(z^*)\right) \right\Vert^2   \nonumber   \\
        &\overset{(a)}{\le}& \frac{\eta^2}{m}\sum_{i=1}^m \left\Vert \sum_{k=0}^{K-1} \left(F_i(z_{i,k})- F_i(z^*)\right)  \right\Vert^2    \nonumber   \\
        &\overset{(b)}{\le}& \frac{\eta^2 K}{m}\sum_{i=1}^m \sum_{k=0}^{K-1}\left\Vert F_i(z_{i,k})-F_i(z^t) + F_i(z^t) - F_i(z^*) \right\Vert^2  \nonumber   \\
        &\overset{(c)}{\le}& \frac{\eta^2 K}{m}\sum_{i=1}^m \sum_{k=0}^{K-1} 2L^2 \left( \left\Vert z_{i,k}-z^t \right\Vert^2 + \left\Vert z^t - z^* \right\Vert^2 \right) \nonumber   \\
        &=& \frac{2\eta^2 L^2 K}{m}\sum_{i=1}^m\sum_{k=0}^{K-1}\left\Vert z_{i,k}-z^t \right\Vert^2 + 2\eta^2 L^2 K^2\left\Vert z^t - z^* \right\Vert^2
    \end{eqnarray}
    where (a) and (b) follow from the relaxed triangle inequality, and (c) follows from Assumption \ref{assp_smooth}.
    
    Then we will derive a bound for $\tau_2$.
    \begin{eqnarray}    \label{proof_bound-inner-prod}
        \tau_2 &=& - 2\eta \langle \frac{1}{m}\sum_{i=1}^m \sum_{k=0}^{K-1}F_i(z_{i,k}), z^t - z^* \rangle \nonumber   \\
        &=& -2\eta \langle \frac{1}{m}\sum_{i=1}^m \sum_{k=0}^{K-1}F_i(z_{i,k}) - F_i(z^t) + F_i(z^t), z^t - z^* \rangle \nonumber   \\
        &=& -2\eta \langle \frac{1}{m}\sum_{i=1}^m \sum_{k=0}^{K-1}F_i(z_{i,k}) - F_i(z^t), z^t - z^* \rangle - 2\eta K \langle F(z^t), z^t - z^* \rangle   \nonumber   \\
        &\overset{(a)}{\le}& 2\eta \left\Vert \frac{1}{m}\sum_{i=1}^m \sum_{k=0}^{K-1}F_i(z_{i,k}) - F_i(z^t) \right\Vert \left\Vert z^t - z^* \right\Vert - 2\eta K \langle F(z^t), z^t - z^* \rangle   \nonumber   \\
        &\overset{(b)}{\le}& \frac{2\eta}{m}\sum_{i=1}^m \sum_{k=0}^{K-1}\Vert F_i(z_{i,k})-F_i(z^t) \Vert \Vert z^t-z^* \Vert - 2\eta K \langle F(z^t), z^t - z^* \rangle   \nonumber   \\
        &\overset{(c)}{\le}& \frac{2\eta L}{m} \sum_{i=1}^m \sum_{k=0}^{K-1}\Vert z_{i,k} - z^t\Vert \Vert z^t-z^*\Vert - 2\eta K \langle F(z^t) - F(z^*), z^t - z^* \rangle \nonumber    \\
        &\overset{(d)}{\le}& \frac{2\eta L}{m} \sum_{i=1}^m \sum_{k=0}^{K-1}\Vert z_{i,k} - z^t\Vert \Vert z^t - z^* \Vert - 2\eta \mu K \Vert z^t - z^* \Vert^2
    \end{eqnarray}
    where (a) follows from the Cauchy-Schwartz inequality; (b) follows from the triangle inequality; (c) follows from Assumption \ref{assp_smooth}; (d) follows from Lemma \ref{lmm_SM}.
    
    From \eqref{proof_bound-incrm-err} and \eqref{proof_bound-inner-prod} we observe that both bounds are relevant to $\Vert z_{i,k} - z^t \Vert$, which  indicates the drift between local models and the global model caused by multiple local updates before the communication. However, this drift can be bounded by the correction techniques of Algorithm \ref{alg_FedLC}:
    \begin{eqnarray}
        \Vert z_{i,k+1} - z^t \Vert &=& \Vert z_{i,k} - z^t - \eta(F_i(z_{i,k}) - F_i(z^t)) - \eta F(z^t) \Vert   \nonumber    \\
        &\le& \Vert z_{i,k} - z^t - \eta (F_i(z_{i,k}) - F_i(z^t)) \Vert + \eta \Vert F(z^t) \Vert  \nonumber   \\
        &\le& \lambda \Vert z_{i,k} - z^t \Vert + \eta \Vert F(z^t) \Vert   \nonumber
    \end{eqnarray}
    for some $0<\lambda<1$ with  $0<\eta<\frac{2\mu}{L^2}$ by Lemma \ref{Lmm_contraction}. It further indicates for any $1\le k \le K$,
    \begin{eqnarray}    \label{proof_bound-drift}
        \Vert z_{i,k} - z^t \Vert &\le& \lambda^k \Vert z_{i,0} - z^t \Vert + \eta k \Vert F(z^t) \Vert  \nonumber   \\
        &\le& \eta K \Vert F(z^t) \Vert \nonumber   \\
        &\le& \eta K L \Vert z^t - z^* \Vert
    \end{eqnarray}
    by noting $z_{i,0}=z^t$ and $\sum_{j=0}^{k-1}\lambda^j \le k$.
    
    Combining \eqref{proof_err}, \eqref{proof_bound-incrm-err} and \eqref{proof_bound-drift} gives
    \begin{eqnarray}
        \Vert z^{t+1} - z^* \Vert^2 &\le& (1+2\eta^2 L^2 K^2 - 2\eta \mu K)\Vert z^t - z^* \Vert^2 + 2(\eta L K)^4 \Vert z^t - z^* \Vert^2 + 2(\eta L K)^2 \Vert z^t - z^* \Vert^2 \nonumber   \\
        &=& \left(1 - 2(\eta\mu K - 2\eta^2 L^2 K^2 - \eta^4 L^4 K^4) \right) \Vert z^t - z^* \Vert^2 .   \nonumber
    \end{eqnarray}

    Let $h(\eta) = 2(\eta\mu K - 2\eta^2 L^2 K^2 - \eta^4 L^4 K^4)$. Given $0<\eta \le \frac{1}{2\mu K}$, $h(\eta) < 1$. Moreover,
    $$
        \frac{h(\eta)}{2\eta} = \mu K - 2\eta L^2 K^2 - \eta^3 L^4 K^4
    $$
    which is a monotonically decreasing function with respect to $\eta$ with $\lim_{\eta \to 0}\frac{h(\eta)}{2\eta} = \mu K > 0$ and $\lim_{\eta \to \infty}\frac{h(\eta)}{2\eta} = -\infty$. Then, we conclude that there exists some $\eta_1 > 0$ such that $h(\eta) > 0$, $\forall 0<\eta<\eta_1$. By defining $\eta_0 = \min \{ 2\mu/L^2, 1/(2\mu K), \eta_1 \}$, it yields $h(\eta) \in (0,1)$, $\forall \eta \in (0,\eta_0)$. Defining $\rho(\eta) = 1 - h(\eta)$ completes the proof.
%\end{proof}

\subsection{Analysis of homogeneous local objectives}   \label{Apx_homo}
In this section, we analyze the convergence properties of FedGDA-GT under homogeneous setting. In fact, when all agents have identical objective functions, i.e., $f_i(x,y) = f(x,y), \forall i=1,\dots,m$, the convergence rate can be at least as $K$ times faster as that in Theorem \ref{Thm_convergence}, which is formally stated by the following proposition:
\begin{proposition}   \label{Prop_homo-convergence}
    Suppose Assumptions \ref{assp_SCSC}, \ref{assp_smooth}, \ref{assp_convex-sets} are satisfied and $f_i(x,y) = f(x,y), \forall i=1,\dots,m$. Let $\{ (x^t, y^t) \}_{t=0}^{\infty}$ be a sequence generated by Algorithm \ref{alg_FedLC}. Then, choosing $\eta = \mu / L^2$, we have
    $$
        \Vert x^t - x^* \Vert^2 + \Vert y^t - y^* \Vert^2 \le (1 - \kappa^{-2})^{Kt} \left( \Vert x^0 - x^* \Vert^2 + \Vert y^0 - y^* \Vert^2 \right), \forall t = 0,1,\dots
    $$
    where $\kappa = L / \mu$ is the condition number of $f(x,y)$. Moreover, we have $1 - \kappa^2 \le \rho(\eta), \forall \eta \in (0, \eta_0)$ where $\eta_0$ is defined in Theorem \ref{Thm_convergence}.
\end{proposition}
\begin{proof}
    As we stated before, Algorithm \ref{alg_FedLC} reduces to conventional GDA under homogeneous setting. Then, for any $l \ge 0$, by the same techniques of Lemma \ref{Lmm_contraction}, we have
    $$
        \Vert z_{l+1} - z^* \Vert^2 \le (1 - 2\eta\mu + \eta^2 L^2)\Vert z_l - z^* \Vert^2 .
    $$
    Setting $\eta = \mu / L^2$, $1 - 2\eta\mu + \eta^2 L^2$ reaches the smallest value, which is $1 - \kappa^{-2}$. Next, we prove that $\kappa^{-2} \ge h(\eta), \forall \eta > 0$. Note that
    $$
        h(\eta) = 2(\eta\mu K - 2\eta^2 L^2 K^2 - \eta^4 L^4 K^4) \le 2\eta(\mu K - \eta L^2 K^2) \le \frac{1}{2} \kappa^{-2} < \kappa^{-2} .
    $$
    Thus, we have $\rho(\eta):= 1 - h(\eta) \ge 1 - \kappa^{-2}, \forall \eta > 0$, which completes the proof.
\end{proof}
To gain the intuition behind Proposition \ref{Prop_homo-convergence}, we note that when $f_i(x,y) = f(x,y)$, $\nabla_x f_i(x,y) = \nabla_x f(x,y)$ and $\nabla_y f_i(x,y) = \nabla_y f(x,y)$. Then local updates in Algorithm \ref{alg_FedLC} reduce to $x^{t+1}_{i,k+1} = x^{t+1}_{i,k} - \eta \nabla_x f(x^{t+1}_{i,k}, y^{t+1}_{i,k})$, similar for $y$. Since at the beginning all agents start at the same point $(x^0, y^0)$, it guarantees that for any $t$, $(x^t_{i,k}, y^t_{i,k}) = (x^t_{j,k}, y^t_{j,k}), \forall i,j\in \{1,\dots,m\}~\mathrm{and}~ \forall k = 0,\dots,K-1$. Thus, Algorithm \ref{alg_FedLC} is equivalent to the centralized GDA in this homogeneous setting, where the global model is improved by $K$ times in one communication round.

%%%%%%%%%%%%%%%%%%%%%%%%%%%%%%%%%%%%%%%%%%%%%%%%%%%%%%%%%%%%%%%
%%%%%%%%%%%%%%%%%%%%%%%%%%%%%%%%%%%%%%%%%%%%%%%%%%%%%%%%%%%%%%%
%%%%%%%%%%%%%%%%%%%%%%%%%%%%%%%%%%%%%%%%%%%%%%%%%%%%%%%%%%%%%%%
%%%%%%%%%%%%%%%%%%%%%%%%%%%%%%%%%%%%%%%%%%%%%%%%%%%%%%%%%%%%%%%
\section{Analysis of generalization properties of minimax learning problems}
In this section, we provide the formal proofs of the results in Section \ref{Sec_bound}. Our proofs are based on the following technical tools.
\begin{definition}
    \textbf{(Growth function)} The growth function $\Pi_{\mathcal{H}}: \mathbb{N} \to \mathbb{N}$ for the hypothesis set $\mathcal{H}$ is defined by
    $$
        \Pi_{\mathcal{H}}(n) = \max_{\xi_1,\dots,\xi_n} \big| \{ (h(\xi_1), \dots, h(\xi_n)) : h \in \mathcal{H} \} \big|
    $$
    where $\xi_1, \dots, \xi_n$ are samples drawn according to some distribution.
\end{definition}

\begin{definition}
    \textbf{(VC-dimension)} The VC-dimension of hypothesis set $\mathcal{H}$ is defined by
    $$
        \mathrm{VCdim}(\mathcal{H}) = \max \{ n : \Pi_{\mathcal{H}}(n) = 2^n \}
    $$
    which measures the size of the largest set of points that can be shattered by $\mathcal{H}$.
\end{definition}

\begin{lemma}
    \textbf{(Massart's lemma)} Let $V \subseteq \mathbb{R}^n$ be a finite set such that $r = \max_{v \in V} \Vert v \Vert$. Then,
    $$
        \mathbb{E}_{\sigma} \left[ \frac{1}{n} \sup_{v \in V} \sum_{j=1}^n \sigma_j v_j \right] \le \frac{r \sqrt{2\log \vert V \vert}}{n}
    $$
    where $v_j$ denotes the $j$th entry of $v$, each $\sigma_j$ is drawn independently from $\{-1, 1 \}$ uniformly.
\end{lemma}

\begin{lemma}
    \textbf{(Sauer's lemma)} Suppose the VC-dimension of hypothesis set $\mathcal{H}$ is $d$. Then for any integer $n \ge d$, 
    $$
        \Pi_{\mathcal{H}}(n) \le \left( \frac{en}{d} \right)^d.
    $$
\end{lemma}

We further introduce McDiarmid's inequality.
\begin{lemma}
    \textbf{(McDiarmid's inequality)} Let $X_1, \dots, X_n$ are independent random variables with $X_i \in \mathcal{X}, \forall i=1,\dots,n$. Suppose there exist some function $f: \mathcal{X}^n \to \mathbb{R}$ and positive scalars $c_1,\dots, c_n$ such that 
    $$
        \big| f(x_1, \dots, x_k, \dots, x_n) - f(x_1, \dots, x'_k, \dots, x_n) \big| \le c_k
    $$
    for all $k=1,\dots,n$ and for any realizations $x_1,\dots,x_n, x'_k \in \mathcal{X}$. Denote $f(X_1,\dots,X_n)$ by $f(S)$. Then, for any $\epsilon>0$,
    \begin{eqnarray}
        \mathbb{P}\left[ f(S) - \mathbb{E}[f(S)] \ge \epsilon \right] &\le& \exp{\left( \frac{-2\epsilon^2}{\sum_{j=1}^n c_j^2} \right)}, \nonumber   \\
        \mathbb{P}\left[ f(S) - \mathbb{E}[f(S)] \le -\epsilon \right] &\le& \exp{\left( \frac{-2\epsilon^2}{\sum_{j=1}^n c_j^2} \right)}. \nonumber
    \end{eqnarray}
\end{lemma}

Then, we are ready to give the proofs of results in Section \ref{Sec_bound}.

\subsection{Proof of Theorem \ref{Thm_general-bound}}
\begin{comment}
We restate Theorem \ref{Thm_general-bound} as follows:
\begin{theorem}
    Suppose $\vert l(x,y;\xi) - l(x,y';\xi) \vert \le L_y \Vert y-y' \Vert$ and $\vert l(x,y;\xi) \vert \le M_i(y)$, $\forall x\in \mathcal{X}$, $\forall y,y' \in \mathcal{Y}^U_{\epsilon}$ and $\forall \xi \sim P_i, i=1\dots,m$ with some positive scalar $L_y$ and real-valued function $M_i(y)>0$. Then, given any $\epsilon > 0$ and $\delta > 0$, with probability at least $1-\delta$ for any $(x,y) \in \mathcal{X}\times \mathcal{Y}$,
    \begin{equation}  
        R(x,y) \le f(x,y) + 2\mathscr{R}(\mathcal{X}, y) + \sqrt{\sum_{i=1}^m \frac{M_i^2(y)}{2m^2 n}\log \frac{\vert \mathcal{Y}_{\epsilon} \vert}{\delta}} + 2L_y \epsilon   .  \nonumber
    \end{equation}
\end{theorem}
\end{comment}

%\begin{proof}
     Let $\mathcal{S} = \{ \mathcal{S}_1, \dots, \mathcal{S}_m \}$ be the collection of all local data sets. Given $y \in \mathcal{Y}$, define
    \begin{equation}
        \Phi(\mathcal{S}) = \sup_{x \in \mathcal{X}} \left\{ R(x,y) - f(x,y) \right\} .  \nonumber
    \end{equation}
    Let $\mathcal{S}' = \{ \mathcal{S}'_1, \dots, \mathcal{S}'_m \}$ be another data collection differing from $\mathcal{S}$ only by point $\xi'_{i,j}$ in $\mathcal{S}'_i$ and $\xi_{i,j}$ in $\mathcal{S}_i$ for some specific $i$. Then,
    \begin{eqnarray}
        \Phi(\mathcal{S}') - \Phi(\mathcal{S}) &=& \sup_{x \in \mathcal{X}}\left\{ R(x,y) - f'(x,y) \right\} - \sup_{x \in \mathcal{X}}\left\{ R(x,y) - f(x,y) \right\}    \nonumber   \\
        &\le& \sup_{x \in \mathcal{X}} \left\{ R(x,y) - f'(x,y) - (R(x,y) - f(x,y)) \right\}    \nonumber   \\
        &=& \sup_{x \in \mathcal{X}} \left\{ f(x,y) - f'(x,y) \right\}  \nonumber   \\
        &=& \sup_{x \in \mathcal{X}} \left\{ \frac{1}{mn} \sum_{j=1}^{n} l(x,y;\xi_{i,j}) - \frac{1}{mn} \sum_{j=1}^{n} l(x,y;\xi'_{i,j}) \right\}   \nonumber   \\
        &\le& \frac{1}{mn} M_i(y).  \nonumber
    \end{eqnarray}

    Applying McDiarmid's inequality gives that for any $c > 0$,
    \begin{eqnarray}
        \mathbb{P}\left[ \Phi(\mathcal{S}) - \mathbb{E}[\Phi(\mathcal{S})] \ge c \right] \le \exp{\left( \frac{-2c^2}{\sum_{i=1}^m \sum_{j=1}^n \left( \frac{M_i(y)}{mn} \right)^2} \right)} = \exp{\left( \frac{-2c^2 m^2 n}{\sum_{i=1}^m M_i^2(y)} \right)}. \nonumber
    \end{eqnarray}
    Setting $\delta = \exp{\left( \frac{-2c^2 m^2 n}{\sum_{i=1}^m M_i^2(y)} \right)}$, we obtain
    $c = \sqrt{\sum_{i=1}^m \frac{M_i^2(y)}{2m^2 n}\log \frac{1}{\delta}}$. Then, with probability at least $1-\delta$, 
    \begin{eqnarray}
        \sup_{x\in \mathcal{X}}\left\{ R(x,y)-f(x,y) \right\} \le \mathbb{E}\left[ \sup_{x\in \mathcal{X}}\left\{ R(x,y)-f(x,y) \right\} \right] + \sqrt{\sum_{i=1}^m \frac{M_i^2(y)}{2m^2 n}\log \frac{1}{\delta}}.   \nonumber
    \end{eqnarray}
    By similar techniques of \cite{FML18}, we have for any $y \in \mathcal{Y}$,
    \begin{eqnarray}
        \mathbb{E}_{\xi \sim P}\left[ \Phi(\mathcal{S}) \right] &=& 
        \mathbb{E}_{\xi \sim P}\left[ \sup_{x\in \mathcal{X}}\left\{ R(x,y)-f(x,y) \right\} \right]  \nonumber   \\
        &=& \mathbb{E}_{\xi \sim P}\left[ \sup_{x \in \mathcal{X}} \left\{ \mathbb{E}_{\xi' \sim P}\left[ f'(x,y) - f(x,y) \right] \right\} \right]    \nonumber  \\
        &\le& \mathbb{E}_{\xi,\xi' \sim P} \left[ \sup_{x \in \mathcal{X}} \{ f'(x,y) - f(x,y) \} \right]  \nonumber   \\
        &=& \underset{\sigma}{\mathbb{E}} \left[ \underset{\xi,\xi' \sim P}{\mathbb{E}}\left[ \sup_{x \in \mathcal{X}}\left\{ \frac{1}{mn} \sum_{i=1}^m\sum_{j=1}^{n} \sigma_{i,j}(l(x,y;\xi'_{i,j}) - l(x,y;\xi_{i,j})) \right\} \right] \right]    \nonumber   \\
        &\le& 2 \underset{\sigma}{\mathbb{E}} \left[ \underset{\xi \sim P}{\mathbb{E}}\left[ \sup_{x \in \mathcal{X}}\left\{ \frac{1}{mn} \sum_{i=1}^m\sum_{j=1}^{n} \sigma_{i,j}l(x,y;\xi_{i,j}) \right\} \right] \right]  \nonumber   \\
        &=& 2 \mathscr{R}(\mathcal{X}, y)   \nonumber
    \end{eqnarray}
    by noting $\xi$ and $\xi'$ are drawn from the same distribution and $\sigma$ is Rademacher variable. Thus, we have given $y \in \mathcal{Y}$, with probability at least $1 - \delta$,
    \begin{eqnarray}
        \sup_{x\in \mathcal{X}}\left\{ R(x,y)-f(x,y) \right\} \le 2\mathscr{R}(\mathcal{X},y) + \sqrt{\sum_{i=1}^m \frac{M_i^2(y)}{2m^2 n}\log \frac{1}{\delta}}.   \nonumber
    \end{eqnarray}
    Since $\mathcal{Y}$ is compact, every open cover of $\mathcal{Y}$ has a finite subcover. Then, we have $\big|\mathcal{Y}_{\epsilon}\big| < \infty$. 
    Taking the union over $\mathcal{Y}_{\epsilon}$, it yields for any $x \in \mathcal{X}$ and $y \in \mathcal{Y}_{\epsilon}$, with probability at least $1-\delta$,
    \begin{eqnarray}
        R(x,y) \le f(x,y) + 2\mathscr{R}(\mathcal{X},y) + \sqrt{\sum_{i=1}^m \frac{M_i^2(y)}{2m^2 n}\log \frac{\vert \mathcal{Y}_{\epsilon} \vert}{\delta}}.   \nonumber
    \end{eqnarray}
    By the definition of $\mathcal{Y}_{\epsilon}$, for any $y \in \mathcal{Y}$, there exists a $y' \in \mathcal{Y}_{\epsilon}$ such that 
    \begin{eqnarray}
        R(x,y) - R(x,y') &\le& L_y \Vert y - y' \Vert \le L_y \epsilon    \nonumber   \\
        f(x, y') - f(x, y) &\le& L_y \Vert y - y' \Vert \le L_y \epsilon    \nonumber      
    \end{eqnarray}
    by Lipschitz continuity of $l$ in $y$.
    Thus, for any $\epsilon > 0$, $x \in \mathcal{X}$ and $y \in \mathcal{Y}$, with probability at least $1-\delta$, the following inequality holds:
    \begin{equation}    \label{proof_complexity-bound-given-y}
        R(x,y) \le f(x,y) + 2\mathscr{R}(\mathcal{X}, y) +  \sqrt{\sum_{i=1}^m \frac{M_i^2(y)}{2m^2 n}\log \frac{\vert \mathcal{Y}_{\epsilon} \vert}{\delta}} + 2L_y \epsilon,      \nonumber
    \end{equation}
    which completes the proof of Theorem \ref{Thm_general-bound}.
%\end{proof}

\subsection{Proof of Corollary \ref{Coro_max-bound}}
\begin{comment}
\begin{corollary} 
    Under the same conditions of Theorem \ref{Thm_general-bound}, with probability at least $1-\delta$ for any $x \in \mathcal{X}$, the following inequality holds for any $\epsilon>0$, $\delta > 0$:
    \begin{equation}
        Q(x) \le g(x) + 2 \mathscr{R}(\mathcal{X},\mathcal{Y}) + \sqrt{\max_{y \in \mathcal{Y}} \left\{ \sum_{i=1}^m \frac{M_i^2(y)}{2m^2 n} \right\} \log \frac{\vert \mathcal{Y}_{\epsilon} , \vert}{\delta}} + 2L_y \epsilon  \nonumber
    \end{equation}
    where $Q(x) = \max_{y \in \mathcal{Y}}R(x,y)$, $g(x) = \max_{y \in \mathcal{Y}} f(x,y)$ are the worst-case population and empirical risks, respectively.
\end{corollary}
\end{comment}

%\begin{proof}
    From Theorem \ref{Thm_general-bound}, it is obvious that with probability at least $1-\delta$ for any $x \in \mathcal{X}$, taking the maximum of $\mathcal{Y}$ gives
    \begin{equation}
        R(x,y) \le \max_{y \in \mathcal{Y}} \left\{ f(x,y) + 2\mathscr{R}(\mathcal{X}, y) + \sqrt{\sum_{i=1}^m \frac{M_i^2(y)}{2m^2 n}\log \frac{\vert \mathcal{Y}_{\epsilon} \vert}{\delta}} \right\} + 2L_y \epsilon .    \nonumber
    \end{equation}
    Since the above inequality holds for any $y \in \mathcal{Y}$, by again taking the maximum over $\mathcal{Y}$ on the left-hand side, we obtain for any $x \in \mathcal{X}$, with probability at least $1-\delta$,
    \begin{eqnarray}
        Q(x) &\le& \max_{y \in \mathcal{Y}} \left\{ f(x,y) + 2\mathscr{R}(\mathcal{X}, y) + \sqrt{\sum_{i=1}^m \frac{M_i^2(y)}{2m^2 n}\log \frac{\vert \mathcal{Y}_{\epsilon} \vert}{\delta}} \right\} + 2 L_y \epsilon    \nonumber   \\
        &\le& g(x) + 2\max_{y \in \mathcal{Y}}\{ \mathscr{R}(\mathcal{X},y) \} + \sqrt{\max_{y \in \mathcal{Y}} \left\{ \sum_{i=1}^m \frac{M_i^2(y)}{2m^2 n} \right\} \log \frac{\vert \mathcal{Y}_{\epsilon} \vert}{\delta}} + 2L_y \epsilon \nonumber
    \end{eqnarray}
    which completes the proof.
    
    \begin{comment}
    Then, we give the proof of statement (b). First, note that 
    $$\mathbb{E}\left[f(x,y)\right] = \mathbb{E}\left[ \frac{1}{mn}\sum_{i=1}^m \sum_{j=1}^n l(x,y;\xi_{i,j}) \right] = \frac{1}{mn}\sum_{i=1}^m \sum_{j=1}^n \mathbb{E}\left[ l(x,y;\xi_{i,j}) \right] = R(x,y)$$
    
    Denoting $x_l$ as the $l$th entry of $x$, applying Hoeffding's inequality immediately gives for any $\epsilon>0$ and given $l \in \{ 1, 2, \dots, p \}$
    \begin{eqnarray}
        \mathbb{P}\left[ \left\vert \frac{\partial f(x,y^*)}{\partial x_l} - \frac{\partial R(x,y^*)}{\partial x_l} \right\vert \ge \epsilon \right] &\le& 2 \exp{\left( \frac{-2\epsilon^2 m^2 n}{\sum_{i=1}^m G_i^2} \right)} \nonumber
    \end{eqnarray}
    By taking the union over all $l$, then with probability at least $1-\delta$ for any $x \in \mathcal{X}$ and any $l \in \{1,\dots,p \}$,
    \begin{equation}
        \left\vert \frac{\partial f(x,y^*)}{\partial x_l} - \frac{\partial R(x,y^*)}{\partial x_l} \right\vert \le \sqrt{\sum_{i=1}^{m} \frac{G_i^2}{2m^2 n} \log \frac{2 p}{\delta}}     \nonumber
    \end{equation}
    Since $\frac{\partial f(x^*, y^*)}{\partial x_l} = 0, \forall l \in \{1,\dots,p \}$ by the definition of $(x^*, y^*)$ and noticing $\Vert \nabla_{x}R(x^*, y^*) \Vert^2 = \sum_{l=1}^{p} \left\vert \frac{\partial R(x^*, y^*)}{\partial x_l} \right\vert^2$, we have
    \begin{equation}
        \Vert \nabla_{x} R(x^*, y^*) \Vert^2 \le \sum_{i=1}^{m} \frac{p G_i^2}{2m^2 n} \log \frac{2 p}{\delta}  \nonumber
    \end{equation}
    Finally, 
    \end{comment}
%\end{proof}

\subsection{Proof of Lemma \ref{Lmm_complexity-bound}}
\begin{comment}
\begin{lemma}  
    Suppose for any $y \in \mathcal{Y}$ and $i \in \{1, \dots, m \}$, $\vert l(\cdot,y;\cdot) \vert$ is bounded by $M_i(y)$ and takes finite number of values. Further, assume that the VC-dimension of $\mathcal{X}$ is $d$. Then, the following inequalities hold:
    \begin{eqnarray}
        \mathscr{R}(\mathcal{X}, \mathcal{Y}) \le \sqrt{2d \max_{y \in \mathcal{Y}} \left\{ \sum_{i=1}^m \frac{M^2_i(y)}{m^2 n} \right\} \left( 1 + \log \frac{mn}{d}\right)} . \nonumber
    \end{eqnarray}
\end{lemma}
\end{comment}

%\begin{proof}
    First, for any fixed $y \in \mathcal{Y}$, define the growth function for the feasible set $\mathcal{X}$:
    \begin{equation}
        \Pi_{\mathcal{X}}^y(N) = \max_{\xi_1,\dots,\xi_N} \big| \{ (l(x,y;\xi_1), \dots, l(x,y;\xi_N)) ~:~ x \in \mathcal{X} \} \big| ,  \nonumber
    \end{equation}
    where $N$ is the total number of samples drawn from the global distribution $P$. Essentially, the growth function $\Pi_{\mathcal{X}}^y (N)$ characterizes that given $y \in \mathcal{Y}$, the maximum number of distinct ways to label $N$ points.
    
    Then for any $y \in \mathcal{Y}$, define a set $V^y$ of vectors in $\mathbb{R}^{mn}$ as
    \begin{equation}
        V^y = \left\{ [l(x,y; \xi_{i,j})] : \xi_{i,j} \sim P,~\forall i=1,\dots,m, ~ j=1,\dots,n \right\} .  \nonumber
    \end{equation}
    For any $v \in V^y$, we have
    $$
        \Vert v \Vert = \sqrt{\sum_{i=1}^m \sum_{j=1}^n \vert l(x,y;\xi_{i,j}) \vert^2} \le \sqrt{\sum_{i=1}^m n M_i^2(y)} .
    $$
    Then, by Massart's lemma, for any $y \in \mathcal{Y}$, it yields
    \begin{eqnarray}    \label{eqn_Rademacher}
        \mathscr{R}(\mathcal{X}, y) \le \sqrt{\sum_{i=1}^m M_i^2(y) \frac{2 \log \vert V^y \vert}{m^2 n}} . \nonumber
    \end{eqnarray}
    Moreover, noting that for any $y \in \mathcal{Y}$,
    $
        \vert V^y \vert \le \Pi_{\mathcal{X}}^y (mn)
    $
    by the definition of $V^y$. Then, by applying Sauer's lemma, we have
    $$
        \Pi_{\mathcal{X}}^y (mn) \le \left( \frac{emn}{d} \right)^d
    $$
    for all $mn \ge d$. By taking the maximum over $\mathcal{Y}$ on both sides of \eqref{eqn_Rademacher}, we directly obtain \eqref{eq_complexity-max-bound}.
%\end{proof}

\section{Code of the experiments}
The datasets and the implementation of the experiments in Section \ref{Sec_exp} can be found through the following link: \href{https://github.com/Starrskyy/FedGDA-GT}{https://github.com/Starrskyy/FedGDA-GT}.


\begin{thebibliography}{99}

\bibitem{GAN}
Ian Goodfellow, Jean Pouget-Abadie, Mehdi Mirza, Bing Xu, David Warde-Farley, Sherjil Ozair, Aaron Courville, and Yoshua Bengio. Generative adversarial nets. \emph{Advances in Neural Information Processing Systems}, 27, 2014.

\bibitem{ImprovedGAN}
Ishaan Gulrajani, Faruk Ahmed, Martin Arjovsky, Vincent Dumoulin, and Aaron C. Courville. Improved training of Wasserstein GANs. \emph{Advances in Neural Information Processing Systems}, 30, 2017.

\bibitem{LSGAN}
Xudong Mao, Qing Li, Haoran Xie, Raymond Y.K. Lau, Zhen Wang, and Stephen Paul Smolley. Least squares generative adversarial networks. \emph{Proceedings of the IEEE International Conference on Computer Vision}, pp. 2794-2802, 2017.

\bibitem{RL17}
Shayegan Omidshafiei, Jason Pazis, Christopher Amato, Jonathan P How, and John Vian. Deep decentralized
multi-task multi-agent reinforcement learning under partial observability. \emph{arXiv preprint arXiv:1703.06182}, 2017.

\bibitem{DuchiAT17}
Aman Sinha, Hongseok Namkoong, and John Duchi. Certifiable distributional robustness with principled adversarial training. \emph{International Conference on Learning Representations}, 2017.

\bibitem{ADICLR18}
Aleksander Madry, Aleksandar Makelov, Ludwig Schmidt, Dimitris Tsipras, and Adrian Vladu. Towards deep learning models resistant to adversarial attacks. \emph{International Conference on Learning Representations}, 2018.

\bibitem{RobustLinear13}
Nam H. Nguyen and Trac D. Tran. Robust lasso with missing and grossly corrupted observations. \emph{ IEEE Transactions on Information Theory}, 59(4): 2036-2058, 2013.

\bibitem{Duchi16}
Hongseok Namkoong and John C. Duchi. Stochastic gradient methods for distributionally robust optimization with f-divergences. \emph{Advances in Neural Information Processing Systems}, 29, 2016.

\bibitem{Duchi17}
Hongseok Namkoong and John C. Duchi. Variance-based regularization with convex objectives. \emph{Advances in Neural Information Processing Systems}, 30, 2017.

\bibitem{Duchi18}
Hongseok Namkoong and John C. Duchi. Learning models with uni- form performance via distributionally robust optimization. \emph{arXiv preprint arXiv:1810.08750}, 2018.

\bibitem{Liang20}
Shiori Sagawa, Pang Wei Koh, Tatsunori B. Hashimoto, and Percy Liang. Distributionally robust neural networks for group shifts: On the importance of regularization for worst-case generalization. \emph{International Conference on Learning Representations}, 2020.

%\bibitem{Hoffman18}
%Judy Hoffman, Mehryar Mohri, and Ningshan Zhang. Algorithms and theory for multiple-source adaptation. \emph{Advances in Neural Information Processing Systems}, 31, 2018.

\bibitem{Zhao18}
Han Zhao, Shanghang Zhang, Guanhang Wu, José M. F. Moura, Joao P. Costeira, and Geoffrey J. Gordon. Adversarial multiple source domain adaptation. \emph{Advances in Neural Information Processing Systems}, 31, 2018.

\bibitem{Mohri19}
Mehryar Mohri, Gary Sivek, and Ananda Theertha Suresh. Agnostic federated learning. \emph{International Conference on Machine Learning}, pp. 4615-4625. PMLR, 2019.

\bibitem{Nedic09}
Angelia Nedic and Asuman Ozdaglar. Subgradient methods for saddle-point problems. \emph{Journal of Optimization Theory and Applications}, pp. 205–228, 2009.

\bibitem{Jordan20}
Tianyi Lin, Chi Jin, and Michael Jordan. On gradient descent ascent for nonconvex-concave minimax problems. \emph{International Conference on Machine Learning}, 119:6083–6093. PMLR, 2020.

\bibitem{TZhang14}
Ohad Shamir, Nathan Srebro, and Tong Zhang. Communication efficient distributed optimization using an approximate Newton-type method. \emph{International Conference on Machine Learning}, 32(2):1000-1008. PMLR, 2014.

\bibitem{Wang17}
Jialei Wang, Mladen Kolar, Nathan Srebro, and Tong Zhang. Efficient distributed learning with sparsity. \emph{International Conference on Machine Learning}, 70:3636-3645. PMLR, 2017.

\bibitem{FedAvg}
Brendan McMahan, Eider Moore, Daniel Ramage, Seth Hampson, Blaise Aguera y Arcas. Communication-efficient learning of deep networks from decentralized data. \emph{International Conference on Artificial Intelligence and Statistics}, 54:1273-1282. PMLR, 2017.

\bibitem{Scaffold}
Sai Praneeth Karimireddy, Satyen Kale, Mehryar Mohri, Sashank Reddi, Sebastian Stich, and Ananda Theertha Suresh. Scaffold: Stochastic controlled averaging for federated learning. \emph{International Conference on Machine Learning}, 119:5132–5143. PMLR, 2020.

\bibitem{FedNova}
Jianyu Wang, Qinghua Liu, Hao Liang, Gauri Joshi, and H. Vincent Poor. Tackling the objective inconsistency problem in heterogeneous federated optimization. \emph{Advances in Neural Information Processing Systems}, 33, 2020.

\bibitem{FedSplit}
Reese Pathak and Martin J. Wainwright. FedSplit: an algorithmic framework for fast federated optimization. \emph{Advances in Neural Information Processing Systems}, 33, 2020.

\bibitem{FedLin}
Aritra Mitra, Rayana Jaafar, George J. Pappas, and Hamed Hassani. Linear convergence in federated learning: tackling client heterogeneity and sparse gradients. \emph{Advances in Neural Information Processing Systems}, 34, 2021.

\bibitem{FedTGAN}
Zilong Zhao, Robert Birke, Aditya Kunar, and Lydia Y. Chen. Fed-TGAN: Federated learning framework for synthesizing tabular data. \emph{arXiv preprint arXiv:2108.07927}, 2021.

\bibitem{BF-FedGAN}
Vaikkunth Mugunthan, Vignesh Gokul, Lalana Kagal, and Shlomo Dubnov. Bias-free FedGAN: A federated approach to generate bias-free datasets. \emph{arXiv preprint arXiv:2108.07927}, 2021.

\bibitem{Deng21}
Yuyang Deng and Mehrdad Mahdavi. Local stochastic gradient descent ascent: convergence analysis and communication efficiency. \emph{International Conference on Artificial Intelligence and Statistics}, 130:1387-1395. PMLR, 2021.

\bibitem{Gauri22}
Pranay Sharma, Rohan Panda, Gauri Joshi, and Pramod K. Varshney. Federated minimax optimization: Improved convergence analyses and algorithms. \emph{arXiv preprint arXiv:2203.04850}, 2022.

\bibitem{Loh11}
Po-Ling Loh and Martin J. Wainwright. High-dimensional regression with noisy and missing data: Provable guarantees with non-convexity. \emph{Advances in Neural Information Processing Systems}, 24, 2011.

\bibitem{RLR11}
Nasser Nasrabadi, Trac Tran, and Nam Nguyen. Robust Lasso with missing and grossly corrupted observations. \emph{Advances in Neural Information Processing Systems}, 24, 2011.

\bibitem{FML18}
Mehryar Mohri, Afshin Rostamizadeh, and Ameet Talwalkar. Foundations of Machine Learning. MIT Press, second edition, 2018.

\bibitem{Tse16}
Farzan Farnia and David Tse. A minimax approach to supervised learning. \emph{Advances in Neural Information Processing Systems}, 29, 2016.

\bibitem{JLee18}
Jaeho Lee and Maxim Raginsky. Minimax statistical learning with Wasserstein distances. \emph{Advances in Neural Information Processing Systems}, 31, 2018.

\bibitem{Farnia21}
Farzan Farnia and Asuman Ozdaglar. Train simultaneously, generalize better: Stability of gradient-based minimax learners. \emph{International Conference on Machine Learning}, 139:3174-3185. PMLR, 2021.

%%%%%%%%%%%%%%%%%%%%%%%%%%%%%%%%%%%%%%%%%%%%%%%%%%%%%%%%%%
%%%% Related work: centralized minimax
%%%%%%%%%%%%%%%%%%%%%%%%%%%%%%%%%%%%%%%%%%%%%%%%%%%%%%%%%%
\bibitem{vNeumann28}
John von Neumann. Zur theorie der gesellschaftsspiele. \emph{Mathematische Annalen}, 100(1): 295–320, 1928.

\bibitem{Robinson51}
Julia Robinson. An iterative method of solving a game. \emph{Annals of Mathematics}, pp. 296–301, 1951.

\bibitem{Sion58}
Maurice Sion. On general minimax theorems. \emph{Pacific Journal of Mathematics}, 8(1):171–176, 1958.

\bibitem{EG76}
G. M. Korpelevich. The extragradient method for finding saddle points and other problems. \emph{Matecon}, 12:747–756, 1976.

\bibitem{Nesterov}
Yurii Nesterov. Dual extrapolation and its applications to solving variational inequalities and related problems. \emph{Mathematical Programming}, 109(2-3):319–344, 2007.

\bibitem{OGDA20}
Aryan Mokhtari, Asuman Ozdaglar, and Sarath Pattathil. A unified analysis of extra-gradient and optimistic gradient methods for saddle point problems: proximal point approach. \emph{International Conference on Artificial Intelligence and Statistics}, 108:1497-1507. PMLR, 2020.

\bibitem{VIGAN18}
Gauthier Gidel, Hugo Berard, Gaëtan Vignoud, Pascal Vincent, and Simon Lacoste-Julien. A variational inequality perspective on generative adversarial networks. \emph{arXiv preprint arXiv:1802.10551}, 2018.

\bibitem{Liu19}
Mingrui Liu, Youssef Mroueh, Jerret Ross, Wei Zhang, Xiaodong Cui, Payel Das, and Tianbao Yang. Towards better understanding of adaptive gradient algorithms in generative adversarial nets. \emph{arXiv preprint arXiv:1912.11940}, 2019.

\bibitem{Lee19}
Maher Nouiehed, Maziar Sanjabi, Tianjian Huang, Jason D. Lee, and Meisam Razaviyayn. Solving a class of non-convex min-max games using iterative first order methods. \emph{Advances in Neural Information Processing Systems}, 32, 2019.

\bibitem{Liu20}
Mingrui Liu, Youssef Mroueh, Wei Zhang, Xiaodong Cui, Tianbao Yang, and Payel Das. Decentralized parallel algorithm for training generative adversarial nets. \emph{Advances in Neural Information Processing Systems}, 33, 2020.

\bibitem{DDJ21}
Jelena Diakonikolas, Constantinos Daskalakis, and Michael I. Jordan. Efficient methods for structured nonconvex-nonconcave min-max optimization. \emph{International Conference on Artificial Intelligence and Statistics}, 130:2746-2754. PMLR, 2021.


%%%%%%%%%%%%%%%%%%%%%%%%%%%%%%%%%%%%%%%%%%%%%%%%%%%%%%%%%%
%%%% Related work: federated minimax
%%%%%%%%%%%%%%%%%%%%%%%%%%%%%%%%%%%%%%%%%%%%%%%%%%%%%%%%%%
\bibitem{Cortes15}
David Mateos-Núnez and Jorge Cortés. Distributed subgradient methods for saddle-point problems. \emph{IEEE Conference on Decision and Control}, pp. 5462–5467, 2015.

\bibitem{BSG20}
Aleksandr Beznosikov, Valentin Samokhin, and Alexander Gasnikov. Distributed saddle-point problems: Lower bounds, near-optimal and robust algorithms. \emph{arXiv preprint arXiv:2010.13112}, 2020.

\bibitem{Xian21}
Wenhan Xian, Feihu Huang, Yanfu Zhang, and Heng Huang. A faster decentralized algorithm for nonconvex minimax problems. \emph{Advances in Neural Information Processing Systems}, 34, 2021.

\bibitem{Rogozin21}
Alexander Rogozin, Aleksandr Beznosikov, Darina Dvinskikh, Dmitry Kovalev, Pavel Dvurechensky, and Alexander Gasnikov. Decentralized distributed optimization for saddle point problems. \emph{arXiv preprint arXiv:2102.07758}, 2021.


\bibitem{DRO-FedAvg}
Yuyang Deng, Mohammad Mahdi Kamani, and Mehrdad Mahdavi. Distributionally robust federated averaging. \emph{Advances in Neural Information Processing Systems}, 33, 2020.

\bibitem{Ali20}
Amirhossein Reisizadeh, Farzan Farnia, Ramtin Pedarsani, and Ali Jadbabaie. Robust federated learning: The case of affine distribution shifts. \emph{Advances in Neural Information Processing Systems}, 33, 2020.

\bibitem{FedGAN}
Mohammad Rasouli, Tao Sun, and Ram Rajagopal. Fedgan: Federated generative adversarial networks for distributed data. \emph{arXiv preprint arXiv:2006.07228}, 2020.

%%%%%%%%%%%%%%%%%%%%%%%%%%%%%%%%%%%%%%%%%%%%%%%%%%%%%%%%%%
%%%% Related work: Generalization
%%%%%%%%%%%%%%%%%%%%%%%%%%%%%%%%%%%%%%%%%%%%%%%%%%%%%%%%%%
\bibitem{Bai18}
Yu Bai, Tengyu Ma, and Andrej Risteski. Approximability of discriminators implies diversity in GANs. \emph{arXiv preprint arXiv:1806.10586}, 2018.

\bibitem{Zhang-bounds18}
Pengchuan Zhang, Qiang Liu, Dengyong Zhou, Tao Xu, and Xiaodong He. On the discrimination-generalization tradeoff in GANs. \emph{arXiv preprint arXiv:1711.02771}, 2017.

\bibitem{Arora17}
Sanjeev Arora, Rong Ge, Yingyu Liang, Tengyu Ma, and Yi Zhang. Generalization and equilibrium in generative adversarial nets (GANs). \emph{International Conference on Machine Learning}, 70:224-232. PMLR, 2017.

\bibitem{Yin19}
Dong Yin, Ramchandran Kannan, and Peter Bartlett. Rademacher complexity for adversarially robust generalization. \emph{International Conference on Machine Learning}, 97:7085-7094, PMLR, 2019.

\bibitem{Loh18}
Justin Khim and Po-Ling Loh. Adversarial risk bounds via function transformation. \emph{arXiv preprint arXiv:1810.09519}, 2018.

\bibitem{Wei19}
Colin Wei and Tengyu Ma. Improved sample complexities for deep networks and robust classification via an all-Layer margin. \emph{arXiv preprint arXiv:1910.04284}, 2019.

\bibitem{Attias19}
Idan Attias, Aryeh Kontorovich, and Yishay Mansour. Improved generalization bounds for robust learning. \emph{International Conference on Algorithmic Learning Theory}, 98:162-183. PMLR, 2019.

\bibitem{Zhang-general-bound21}
Junyu Zhang, Mingyi Hong, Mengdi Wang, and Shuzhong Zhang. Generalization bounds for stochastic saddle point problems. \emph{International Conference on Artificial Intelligence and Statistics}, 130:568-576. PMLR, 2021.

\bibitem{Jin-minimax}
Chi Jin, Praneeth Netrapalli, and Michael I. Jordan. What is local optimality in nonconvex-nonconcave minimax optimization? \emph{arXiv preprint arXiv:1902.00618}, 2019.


\end{thebibliography}
\end{document}